\documentclass[a4paper,10pt]{article}
\usepackage[english]{groreport}
\usepackage{graphicx}
\usepackage{amssymb}
\usepackage{amsmath}
\usepackage{mathrsfs}
\usepackage{mathabx} 
\usepackage{shuffle}
\usepackage{algorithm}
\usepackage{algorithmic}
\usepackage{booktabs} 
\usepackage{wrapfig}
\usepackage{listings} 
\usepackage{xcolor}
\usepackage{color}
\usepackage{comment}

\usepackage{fancyhdr}
\usepackage[latin1]{inputenc} 
\usepackage{placeins} 
\usepackage{float}
\usepackage{floatflt}
\usepackage{subfloat}
\usepackage{theorem}
\usepackage{relsize}

\usepackage{caption}
\usepackage{subcaption}
\usepackage{epstopdf}
\usepackage{QFHeader}

\usepackage{bbm}

\usepackage{pgfplots}
\usepackage{tikz}
\usepackage{tikzscale}
\usetikzlibrary{shapes.multipart}
\usetikzlibrary{shapes.geometric, arrows, calc, automata, petri, shapes, decorations.pathmorphing, decorations.pathreplacing, positioning, fit, backgrounds}

\newtheorem{definition}{\textbf{Definition}}[section]
 
\newtheorem{theorem}{\textbf{Theorem}}[section]

\newtheorem{proposition}{\textsl{\textbf{Proposition}}}
 
\newtheorem{lemma}{\textbf{Lemma}}[section]
 
\newtheorem{remark}{\textsl{\textbf{Remark}}}[section]

\newtheorem{proof}{\textbf{proof}}[section]

\newcommand{\be}{\begin{equation}}
\newcommand{\en}{\end{equation}}
\newcommand{\beq}{\begin{equation}}
\newcommand{\eeq}{\end{equation}}
\newcommand{\bes}{\begin{eqnarray*}}
\newcommand{\ens}{\end{eqnarray*}}
\newcommand{\beqa}{\begin{eqnarray*}}
\newcommand{\eeqa}{\end{eqnarray*}}
\newcommand{\bea}{\begin{eqnarray}}
\newcommand{\ena}{\end{eqnarray}}

\newcommand{\xuparrow}[1]{%
  {\left\uparrow\vbox to #1{}\right.\kern-\nulldelimiterspace}
}

\newcommand{\xdownarrow}[1]{%
  {\left\downarrow\vbox to #1{}\right.\kern-\nulldelimiterspace}
}

\DeclareMathAlphabet{\mathsfit}{\encodingdefault}{\sfdefault}{m}{sl}
\SetMathAlphabet{\mathsfit}{bold}{\encodingdefault}{\sfdefault}{bx}{n}

\title{\textbf{Generative Path-Law Jump-Diffusion: Sequential MMD-Gradient Flows and Generalisation Bounds in Marcus-Signature RKHS} 
}

\author{Daniel \textsc{Bloch}  
\footnote{Visting Professor at the College of Engineering and Computer Science, VinUniversity, Hanoi.} \\
\textsc{University of Paris 6 \& VinUniversity} \\
Working Paper \\
\footnote{All mistakes are ours.} 
}

\date{27th of February 2026 \\
Version : 1.0.0}

\begin{document}
\QFHeader{\textbf{Generative Path-Law Jump-Diffusion: Sequential MMD-Gradient Flows and Generalisation Bounds in Marcus-Signature RKHS}}{Daniel Bloch}{27th of February 2026}
\maketitle

\begin{abstract}
This paper introduces a novel generative framework for synthesising forward-looking, c\`adl\`ag stochastic trajectories that are sequentially consistent with time-evolving path-law proxies, thereby incorporating anticipated structural breaks, regime shifts, and non-autonomous dynamics. By framing path synthesis as a sequential matching problem on restricted Skorokhod manifolds, we develop the \textit{Anticipatory Neural Jump-Diffusion} (ANJD) flow, a generative mechanism that effectively inverts the time-extended Marcus-sense signature. Central to this approach is the Anticipatory Variance-Normalised Signature Geometry (AVNSG), a time-evolving precision operator that performs dynamic spectral whitening on the signature manifold to ensure contractivity during volatile regime shifts and discrete aleatoric shocks. We provide a rigorous theoretical analysis demonstrating that the joint generative flow constitutes an infinitesimal steepest descent direction for the Maximum Mean Discrepancy functional relative to a moving target proxy. Furthermore, we establish statistical generalisation bounds within the restricted path-space and analyse the Rademacher complexity of the whitened signature functionals to characterise the expressive power of the model under heavy-tailed innovations. The framework is implemented via a scalable numerical scheme involving Nystr\"om-compressed score-matching and an anticipatory hybrid Euler-Maruyama-Marcus integration scheme. Our results demonstrate that the proposed method captures the non-commutative moments and high-order stochastic texture of complex, discontinuous path-laws with high computational efficiency.
\end{abstract}

\medskip
\noindent\textbf{Keywords:} Anticipatory Neural Jump-Diffusion (ANJD), Marcus-Sense Signature, Skorokhod Space, MMD-Gradient Flow, Adaptive Variance-Normalised Signature Geometry (AVNSG), Schr\"odinger Bridge, Euler-Maruyama-Marcus (EMM) Integration, Nystr\"om Approximation, Stochastic Synthesis, Spectral Whitening.

\section{Introduction}
\label{sec_introduction}

\subsection{High-level goal}

The primary objective of this work is to establish a rigorous generative framework for the synthesis of forward-looking, c\`adl\`ag stochastic trajectories; by enforcing sequential consistency with time-evolving path-law proxies, the model natively incorporates expected structural breaks, regime shifts, and evolving volatility patterns into the generative process. While previous advancements in signature-based filtering have enabled the recursive estimation of expected path-dynamics, the inversion of these abstract, infinite-dimensional moments into concrete, synthetic realisations on the restricted Skorokhod manifolds $\mathcal{D}_s$, particularly those containing discrete discontinuities, remains a formidable challenge.

\medskip
\noindent
This paper seeks to bridge this gap by treating the generative task as a sequential anticipatory transport problem within the Skorokhod space, equipped with a time-varying signature-based metric. Our goal is to develop the \textit{Anticipatory Neural Jump-Diffusion} (ANJD) flow, a mechanism that inverts the time-extended Marcus-sense signature Bochner integral (Marcus ~\cite{Marcus81}, Yosida ~\cite{Yosida95}) to produce an ensemble of paths whose collective law $\mu_s$ is infinitesimally coerced toward a moving target proxy $\hat{\Phi}_{s|t}$ for $s \in [t, t+\tau]$. 

\medskip
\noindent
Crucially, we justify the representational sufficiency of the signature in this discontinuous setting by appealing to recent universal approximation results for c\`adl\`ag paths (Cuchiero et al. ~\cite{CuchieroEtAl25}), which demonstrate that linear functionals of the time-extended signature can uniformly approximate any continuous functional on the Skorokhod manifold. Furthermore, following Friz et al. ~\cite{FrizEtAl17,FrizEtAl18}, we treat these jump-diffusions as L\'evy rough paths, ensuring that the Marcus-sense signature remains a group-valued descriptor that uniquely characterises the path-law. By leveraging a time-varying, precision-weighted geometry (AVNSG), we ensure that the resulting synthesis maintains high-fidelity stochastic texture, capturing non-linear dependencies and non-commutative higher-order moments even in the presence of significant non-stationarity and forecasted aleatoric shocks.

\subsection{Motivation and literature positioning}

The generative modelling of high-frequency, non-stationary stochastic processes remains a critical frontier in quantitative finance and physical sciences (Caulfield et al. ~\cite{CaulfieldEtAl24}). Traditional architectures, such as TimeGAN (Yoon et al. ~\cite{YoonEtAl19}), FinGAN (Vuleti\'c et al. ~\cite{VuleticEtAl24}), or Variational Autoencoders (VAEs) (Buehler et al. ~\cite{BuehlerEtAl20}), often struggle to maintain the path-geometric integrity required to capture higher-order dependencies, such as leverage effects and volatility clusters, especially when the underlying law undergoes abrupt regime shifts or exhibits discrete structural breaks. While Neural SDEs (Li et al. ~\cite{LiEtAl20}, Kidger et al. ~\cite{KidgerEtAl21}) and Diffusion models (Ho et al. ~\cite{HoEtAl20}) have provided a robust continuous-time framework for path-generation, they frequently lack a structural mechanism to anchor the synthesis to a rigorous, infinite-dimensional representation of the conditional path-law in the presence of jump-discontinuities.

\medskip
\noindent
This work is positioned at the intersection of path-signature theory (Lyons et al. ~\cite{LyonsEtAl07,LyonsEtAl11,LyonsEtAl22}, Chevyrev et al. ~\cite{ChevyrevEtAl16}) and generative stochastic transport (Elworthy ~\cite{Elworthy82}, Chen et al. ~\cite{ChenEtAl16}). We build directly upon the recursive filtering framework established in Bloch ~\cite{Bloch26a,Bloch26b}, which utilises the signature of the observational filtration to track a latent proxy in the Signature RKHS. While recent literature has explored the use of signatures as loss functionals in GAN-based settings (Liao et al. ~\cite{LiaoEtAl20}, Issa et al. ~\cite{IssaEtAl23}, Bayer et al. ~\cite{BayerEtAl26}), these approaches often treat the signature as a static descriptor of continuous paths. 

\medskip
\noindent
In contrast, our framework leverages the \textit{expected Marcus signature} as a dynamic target within a jump-diffusion Schr\"odinger Bridge formulation. By introducing the \textit{Anticipatory Neural Jump-Diffusion} (ANJD) and the Adaptive Variance-Normalised Signature Geometry (AVNSG) (Bloch ~\cite{Bloch25a,Bloch25b}), we extend the literature on signature kernels to c\`adl\`ag environments. This provides a metric-driven approach to spectral whitening that ensures the generative flow remains stable under heavy-tailed innovations and heteroskedastic shocks, explicitly accounting for the non-commutative nature of discrete jumps in the Skorokhod space.

\subsection{Main contributions}

The primary contributions of this paper are summarised as follows:

\begin{itemize}
    \item \textbf{Sequential Anticipatory Flow Framework:} We introduce the \textit{Anticipatory Neural Jump-Diffusion} (ANJD) architecture, a novel generative paradigm that bridges recursive filtering and path synthesis. By conditioning a non-Markovian Jump-SDE on a time-evolving path-law proxy $\hat{\Phi}_{s|t}$, we enable the sequential matching of c\`adl\`ag trajectories on restricted Skorokhod manifolds $\mathcal{D}_s$, ensuring consistency with forecasted structural breaks and non-autonomous regime shifts.
    
    \item \textbf{Theoretical Foundation of Infinitesimal MMD Flows:} We establish that the generative drift and jump intensity constitute the infinitesimal steepest descent direction for the Maximum Mean Discrepancy (MMD) functional relative to a moving target proxy. We provide a rigorous proof (Theorem~\ref{thm:mmd_minimisation}) linking the infinitesimal generator of the ANJD process to the continuous minimisation of path-law discrepancy.
    
    \item \textbf{Adaptive Variance-Normalised Signature Geometry (AVNSG):} We define a time-evolving precision operator $\mathcal{Q}_s$ that performs dynamic spectral whitening on the $(d+1)$-dimensional signature manifold. This geometry ensures stability under forecasted volatility explosions and provides a mechanism to prioritise the matching of principal structural modes during the flow.
    
    \item \textbf{Statistical Generalisation in Restricted Spaces:} We derive high-probability bounds for the generalisation error of the empirical expected signature within the $\mathcal{D}_s$ topology (Theorem~\ref{thm:generalisation_bound}). We further characterise the expressive power of the model via the Rademacher complexity of whitened signature functionals, providing explicit bounds that scale with the spectral radius of the moving AVNSG operator.
    
    \item \textbf{Scalable Implementation via Dynamic Nystr\"om Updates:} We present an efficient numerical scheme utilising Nystr\"om-compressed score-matching and $O(m^2)$ rank-1 precision updates. This allows the model to propagate the infinite-dimensional geometry through both continuous diffusion and discrete jump-discontinuities by tracking the innovation in the signature kernel feature map.
\end{itemize}

\subsection{Organisation of the paper}

The remainder of the paper is organised as follows. Section (\ref{sec:math_foundations}) establishes the mathematical foundations of path-law embeddings in the signature RKHS and introduces the AVNSG precision operator for spectral whitening. 
Section (\ref{sec:generative_dynamics}) details the construction of the Anticipatory Generative Flow, framing the synthesis task as an \textit{Anticipatory Neural Jump-Diffusion} (ANJD) process. We formalise the path evolution as a sequential matching problem on restricted Skorokhod manifolds $\mathcal{D}_s$, where the drift, diffusion, and jump intensity are dynamically regulated by the velocity of the moving target proxy $\hat{\Phi}_{s|t}$.
In Section (\ref{sec:mmd_gradient_flows}), we provide the theoretical justification for the generative drift, proving its optimality as an infinitesimal steepest descent direction in the MMD sense. Section (\ref{sec:generalisation}) derives the statistical generalisation bounds and complexity results for the whitened signature functionals within the time-evolving geometry. In Section (\ref{sec:implementation}), we detail the practical implementation of the model through joint signature score-matching and an anticipatory Euler-Maruyama-Marcus (EMM) integration scheme, utilising dynamic Nystr\"om-compressed updates to propagate the coupled jump-geometry system in $O(m^2)$ complexity.

\section{Mathematical foundations}
\label{sec:math_foundations}

In this section, we formalise the representation of probability measures over path-space as elements of the Signature RKHS and define the adaptive geometry required for non-stationary transport.

\subsection{Preliminaries: Recursive filtering and latent propagation}

We establish our framework on a complete probability space $(\Omega, \mathcal{F}, \mathbb{P})$ supporting an $\mathbb{F}$-adapted semimartingale $X$. In practice, we operate under the \textbf{observational filtration} $\mathbb{A} = (\mathcal{A}_t)_{t \ge 0}$, where $\mathcal{A}_t = \sigma(\{ (t_i, X_{t_i}, M_{t_i}) : t_i \le t \}) \subset \mathcal{F}_t$, representing the information set of irregularly sampled and masked observations. To handle this discrete data while maintaining a continuous causal structure, we utilise the \textbf{rectilinear interpolation} scheme $\tilde{X}^{\le t}$, which ensures the observed history is a continuous process of bounded variation. For notational simplicity in the subsequent sections, we shall denote the rectilinear interpolation $\tilde{X}^{\le t}$ simply as $X_t$.

\medskip
\noindent
Following Bloch ~\cite{Bloch26a,Bloch26b}, the state of the system is characterised by a \textit{conditional path-law proxy} $\Phi_{t|\mathcal{A}_t} \in \mathcal{H}_{sig}$, representing the expected signature of the process conditioned on the observational filtration $\mathcal{A}_t$.

\begin{definition}[Filtered Proxy and Jump-Flow Latent Propagation]
The proxy $\Phi_{t|\mathcal{A}_t}$ is recovered from a latent state $Z_t \in \mathcal{Z}$ via a tensorial readout map $\mathcal{P}_\theta$. The latent state $Z_t$ is a hidden controller governed by a Jump-Flow Controlled Differential Equation (CDE) that reconciles continuous drift with discrete information shocks:
\begin{equation}
    \Phi_{t|\mathcal{A}_t} = \mathcal{P}_\theta(Z_t), \quad dZ_t = f_\theta(Z_{t-}, \pi_r(X)) dt + (\rho_\theta(Z_{t-}, X_t, M_t) - Z_{t-}) dN_t,
\end{equation}
where $f_\theta$ is the continuous flow vector field, $\rho_\theta$ is the discrete rectification operator triggered by the counting process $N_t$, and $\pi_r(X)$ is the truncated signature of the path history.
\end{definition}

\medskip
\noindent
For out-of-sample synthesis, the latent state is sequentially extrapolated across the future horizon $s \in [t, t+\tau]$. In the absence of new observations ($\Delta N_u = 0$ for $u > t$), the estimator anticipates the evolution of the latent geometry by integrating the non-autonomous continuous flow, resulting in a time-evolving path-law proxy that tracks the infinitesimal deformation of the signature manifold.

\begin{definition}[Anticipatory Latent Propagation]
Given the observational filtration $\mathcal{A}_t$ and a forward path extension $\hat{X}_{t:s}$, the time-evolving path-law proxy $\hat{\Phi}_{s|t}$ is defined as the push-forward of the latent state $Z_s$ through the topological embedding $\mathcal{P}_\theta$:
\begin{equation}
    \hat{\Phi}_{s|t} = \mathcal{P}_\theta(Z_s), \quad Z_s = Z_t + \int_t^s F_\theta(Z_u, \hat{\Phi}_{u|t}) \, d\hat{X}_u
\end{equation}
where $F_\theta$ denotes the operator-valued generator of the Neural CDE drift for $s \in [t, t+\tau]$.
\end{definition}

\begin{remark}[Historical Reconstruction]
While the primary focus of this framework is the anticipatory synthesis of future trajectories, the formulation is natively symmetric with respect to the temporal direction. Specifically, the same generative mechanism can be applied to historical reconstruction or "in-sample" synthesis within a range $[t-\tau, t]$. In such cases, the latent state is conditioned on the observed filtration $\mathcal{A}_s$ and the realised path $X_{t-\tau:s}$, where the moving target becomes the filtered path-law proxy $\Phi_{s|\mathcal{A}_s}$ for $s \in [t-\tau, t]$. This dual capability ensures that the ANJD flow can be utilised both as a predictive engine for future aleatoric shocks and as a high-fidelity structural interpolator for historical data, maintaining consistency with the time-evolving signature geometry across any arbitrary sub-interval of the Skorokhod manifold.
\end{remark}

\subsection{Synthesis of the anticipatory path-drift}

The forward path extension $\hat{X}_{t:t+\tau}$ provides the necessary control for the predictive flow across the restricted Skorokhod manifolds $\mathcal{D}_s$. In this framework, the generated future path $\hat{X}_{t:s}$ is synthesised by a deterministic neural architecture, typically denoted as the \textbf{actor} or \textbf{forecaster} $\mu_\theta$. This network serves as a generative mapping that ingests the current filtered latent state $Z_t$ and its associated tensorial proxy $\hat{\Phi}_{t|t}$ to output a sequence of predicted increments $d\hat{X}_u$ across the future horizon $u \in [t, s]$. Conceptually, this construction represents the agent's \textit{ex-ante} "best guess" or imagined trajectory, providing the necessary physical grounding to evaluate the self-consistency of the underlying signature flow against the anticipated latent evolution.

\medskip
\noindent
Formally, the infinitesimal increments of the anticipated path are governed by the deterministic mapping $\mu_\theta$:
\begin{equation}
    d\hat{X}_u = \mu_\theta(u, Z_t, \hat{\Phi}_{t|t}) du, \quad u \in [t, s],
\end{equation}
such that the integrated future trajectory is recovered as:
\begin{equation}
    \hat{X}_s = X_t + \int_t^s \mu_\theta(u, Z_t, \hat{\Phi}_{t|t}) du.
\end{equation}
This drift serves as the control input for the latent propagation, ensuring that the evolution of the signature manifold is tied to a concrete, albeit synthetic, realisation of the process.

\subsection{The signature Bochner integral and path-law embeddings}

Let $\mathcal{D} = \mathcal{D}([0, T], \mathbb{R}^d)$ denote the Skorokhod space of c\`adl\`ag paths, and let $\mathcal{P}(\mathcal{D})$ be the set of Borel probability measures on $\mathcal{D}$. To ensure a universal and injective representation for jump-diffusions, we consider the time-extended path $\tilde{\gamma}_s = (s, \gamma_s)$, which embeds the temporal evolution directly into the path geometry.

\begin{definition}[Signature Mean Embedding]
For a probability measure $\mu \in \mathcal{P}(\mathcal{D})$, the \textit{path-law proxy} $\Phi_\mu \in \mathcal{H}_{sig}$ is defined as the signature Bochner integral of the time-extended Marcus-signature map $S: \mathcal{D} \to \mathcal{H}_{sig}$ over the realised paths $\gamma \in \mathcal{D}$:
\begin{equation}
    \Phi_\mu := \mathbb{E}_{\mu}[S(\tilde{\gamma})] = \int_{\mathcal{D}} S(\tilde{\gamma}) \, d\mu(\gamma).
\end{equation}
Following Yosida \cite{Yosida95}, this construction ensures that the expected signature is the unique element in the tensor algebra $\mathcal{H}_{sig}$ such that for any linear functional $L \in \mathcal{H}_{sig}^*$, the relation $L(\Phi_\mu) = \mathbb{E}_\mu[L(S(\tilde{\gamma}))]$ holds, providing a rigorous foundation for the inversion of path-laws from their non-commutative moments.
\end{definition}

\begin{remark}[Transition from Filtered to Generative Proxy]
While the filtered proxy $\Phi_{t|\mathcal{A}_t}$ introduced in preceding work (Bloch \cite{Bloch26a,Bloch26b}) serves as a retrospective point-estimate, summarising the expected path-dynamics given historical observations, the generative embedding $\Phi_\mu$ functions as a canonical representative of the conditional path-measure $\mu$. In this prospective context, $\Phi_\mu$ is treated as a moment-generating element in $\mathcal{H}_{sig}$ that uniquely characterises the distributional flow. The generative task is thus framed as the inversion of this signature Bochner integral, where we seek to synthesise an ensemble of trajectories whose collective signature moments coincide with the target proxy under the AVNSG metric.
\end{remark}

\begin{proposition}[Injectivity and Universal Approximation]
\label{pro:injectivity_characteristic_property}
The time-extended signature map $S$ is a universal and characteristic kernel on the space of c\`adl\`ag paths. Following Cuchiero et al. \cite{CuchieroEtAl25}, the inclusion of the time component ensures that the embedding $\mu \mapsto \Phi_\mu$ is injective on $\mathcal{P}(\mathcal{D})$. Furthermore, linear functionals of the signature can uniformly approximate any continuous functional on compact subsets of the Skorokhod space, justifying the use of $\Phi_\mu$ as a sufficient statistic for the law of jump-diffusions.
\end{proposition}
See proof in Appendix (\ref{app:proof_injectivity_characteristic_property}). \\

\subsection{AVNSG metric spaces and spectral whitening}

To account for local heteroskedasticity and the non-uniform temporal distribution of jumps, we equip the Hilbert space with a time-varying metric derived from the infinitesimal variations of the time-extended signature.

\begin{definition}[Adaptive Precision Operator]
Let $\tilde{S}_t$ be the time-extended Marcus signature. Let $\Omega_t \in \mathcal{L}(\mathcal{H}_{sig})$ be the Long-Run Covariance (LRC) operator of the signature process, capturing the second-order statistics of the augmented path increments $(dt, dX_t)$. The \textit{AVNSG Precision Operator} $\mathcal{Q}_t$ is defined via the regularised inverse:
\begin{equation}
    \mathcal{Q}_t := (\Omega_t + \lambda I)^{-1}, \quad \lambda > 0.
\end{equation}
The induced AVNSG inner product is given by $\langle u, v \rangle_{\mathcal{Q}_t} = \langle u, \mathcal{Q}_t v \rangle_{\mathcal{H}_{sig}}$, defining a geometry where features, including temporal duration and jump magnitudes, are asymptotically decorrelated and variance-normalised.
\end{definition}

\medskip
\noindent
By incorporating the temporal coordinate into the LRC, $\mathcal{Q}_t$ effectively weights the relevance of path-dependent features relative to the intensity of the underlying L\'evy measure. In regions of high jump frequency, the metric compresses the importance of individual increments, whereas in quiescent periods, the precision operator amplifies the significance of the "drift" component, ensuring a consistent gradient signal for the generative flow.

\subsection{Kernel herding in tensor algebra}

The transition from the proxy $\Phi_\mu$ to representative sample paths is governed by the minimisation of the Maximum Mean Discrepancy (MMD) on the time-extended signature manifold.

\begin{lemma}[Greedy Path Reconstruction]
\label{lem:greedy_path_reconstruction}
Given a target proxy $\Phi^*$, a sequence of Dirac measures $\delta_{\gamma_i}$ is generated via the inductive herding rule over the space of time-extended paths $\tilde{\gamma} = (s, \gamma_s)$:
\begin{equation}
    \gamma_{k+1} = \arg\max_{\gamma \in \mathcal{X}} \left\langle \Phi^* - \frac{1}{k} \sum_{i=1}^k S(\tilde{\gamma}_i), \, S(\tilde{\gamma}) \right\rangle_{\mathcal{Q}_t}.
\end{equation}
The empirical average of the time-extended signatures $\hat{\Phi}_k = \frac{1}{k} \sum_{i=1}^k S(\tilde{\gamma}_i)$ converges to the target $\Phi^*$ in the $\mathcal{Q}_t$-norm at a rate of $\mathcal{O}(1/k)$, ensuring that the reconstructed ensemble captures the non-commutative moments and temporal evolution of the underlying measure.
\end{lemma}
See proof in Appendix (\ref{app:proof_greedy_path_reconstruction}). \\

\section{Generative path-law dynamics}
\label{sec:generative_dynamics}

This section details the transition from the recursive filtering of the latent proxy to the synthesis of sample paths via a conditioned stochastic flow.

\subsection{The VJF-encoder and latent initialisation}

The filtered latent state $Z_t \in \mathcal{Z}$ from the VJF-Kernel serves as a compressed representation of the filtration $\mathcal{A}_t$. We define the encoding process that bridges the filtering manifold to the generative path-space.

\begin{definition}[Manifold-Conditioned Initialisation]
Let $\mathcal{E}_\theta: \mathcal{Z} \to \mathbb{R}^d$ be a learned encoding map. The generative process for a future horizon $s \in [t, t+\tau]$ is initialised at the current filtered observation $X_t$, with the drift dynamics conditioned on the latent proxy:
\begin{equation}
    X_{t|t} = X_t, \quad V_t = \mathcal{E}_\theta(Z_t).
\end{equation}
The vector $V_t$ encapsulates the local velocity and curvature constraints inherited from the historical path-geometry.
\end{definition}

\begin{remark}[Readout vs. Encoding Maps]
It is critical to distinguish the encoding map $\mathcal{E}_\theta$ from the tensorial readout map $\mathcal{P}_\theta$ utilised in the filtering stage. While $\mathcal{P}_\theta: \mathcal{Z} \to \mathcal{H}_{sig}$ recovers the global coordinate-free representation of the path-law proxy, the encoding map $\mathcal{E}_\theta: \mathcal{Z} \to \mathbb{R}^d$ performs a local projection back into the physical tangent space. This ensures that the generative SDE is seeded with initial conditions, such as instantaneous velocity and local trend, that are consistent with the latent manifold's geometry, effectively bridging the abstract Hilbert space with the concrete path-space realisation.
\end{remark}

\subsection{The anticipatory path-SDE}

The evolution of the synthetic trajectories is governed by an \textit{Anticipatory Neural Jump-Diffusion} (ANJD) process, where the drift, diffusion, and jump intensity are explicitly regularised by the clock $s$, the forecasted path-law proxy $\hat{\Phi}_{s|t}$, and the adaptive geometry $\mathcal{Q}_s$.

\begin{definition}[Anticipatory Generative Flow]
Let $(\Omega, \mathcal{F}, \{\mathcal{F}_s\}_{s \ge t}, \mathbb{P})$ be a filtered probability space. The generative path $X_s$ for $s \in [t, t+\tau]$ is defined as the unique c\`adl\`ag solution to the following time-augmented path-dependent Jump-SDE:
\begin{equation} 
\label{eq:anticipatory_generative_flow}
    dX_s = f_\theta(s, X_s, \hat{\Phi}_{s|t}) \, ds + g_\theta(s, X_s, \hat{\Phi}_{s|t}) \, \diamond dW_s + h_\theta(s, X_{s-}, \hat{\Phi}_{s|t}) \, dN_s
\end{equation}
where $\diamond$ denotes the Marcus integration (to ensure the solution remains on the appropriate manifold), $W_s$ is a $d$-dimensional $\mathcal{F}_s$-Wiener process, and $N_s$ is a non-homogeneous Poisson process with an $\mathcal{F}_s$-predictable intensity $\lambda_s = \lambda_\theta(s, X_{s-}, \hat{\Phi}_{s|t})$. The model parameters $\theta = \{ \theta_f, \theta_g, \theta_h, \theta_\lambda \}$ parameterise the drift $f$, diffusion $g$, jump-amplitude $h$, and intensity $\lambda$, respectively. 
We assume the coefficients $f, g, h$ satisfy the required Lipschitz and linear growth conditions in their spatial arguments to ensure the existence of a unique strong solution. The continuous part of Eq. (\ref{eq:anticipatory_generative_flow}) is interpreted in the Marcus sense to ensure the signature remains group-valued across discontinuities.
\end{definition}

\begin{proposition}[Structural Coupling and Jump-Aware Dynamics]
\label{pro:structural_coupling_and_dynamics}
The Anticipatory Generative Flow defined in Eq. (\ref{eq:anticipatory_generative_flow}) constitutes a novel class of Neural Jump-SDEs characterised by the following properties:
\begin{enumerate}
    \item \textbf{C$^1$-Boundary Consistency:} The drift $f_\theta$ is constrained by the initial boundary condition $f_\theta(t, X_t, \hat{\Phi}_{t|t}) = V_t$, ensuring first-order continuity between the historical trajectory and the generated flow at the junction $s=t$.
    \item \textbf{Polynomial Tractability and Universality:} Following Cuchiero et al. \cite{CuchieroEtAl25}, we justify the coupling $(f_\theta, g_\theta, h_\theta, \lambda_\theta)$ to the signature proxy $\hat{\Phi}_{s|t}$ and the clock $s$ by noting that L\'evy-type signature models are polynomial processes on the extended tensor algebra. This ensures that the law of the process can be evolved and "pushed" by linear functionals of the time-extended signature, providing a universal representation for any continuous functional of c\`adl\`ag paths.
    \item \textbf{Infinitesimal Signature Matching:} The drift $f_\theta$ is functionally coupled to the latent path-law proxy such that the expected infinitesimal signature of the ensemble aligns with the tangent of the push-forward mapping in the RKHS. Specifically, the drift satisfies the differential matching:
    \begin{equation}
        d \mathbb{E}_{\mu_s}[S(s, X_s)] \approx \nabla_s \hat{\Phi}_{s|t} \, ds = \nabla_{Z_s} \mathcal{P}_\theta \cdot F_\theta(Z_s, \hat{\Phi}_{s|t}) \, d\hat{X}_s,
    \end{equation}
    where $\nabla_{Z_s} \mathcal{P}_\theta$ is the Jacobian of the topological embedding, ensuring the flow reacts to the manifold dynamics of the latent state $Z_s$.
    \item \textbf{Discontinuous Structural Breaks:} The inclusion of the $\mathcal{F}_s$-predictable intensity $\lambda_s = \lambda_\theta(s, X_{s-}, \hat{\Phi}_{s|t})$ enables the flow to exhibit jump-discontinuities. This allows the model to trigger endogenous "shocks" or regime shifts that are structurally conditioned on the absolute time and the anticipated geometry of the path-law.
    \item \textbf{Non-Gaussianity and Tail Risk:} The joint non-linear dependence of the diffusion $g_\theta$ and jump-amplitude $h_\theta$ on $(s, \hat{\Phi}_{s|t})$ allows the transition densities to capture extreme kurtosis and heavy-tailed innovations, providing a mechanism for modeling black-swan events consistent with the signature manifold.
    \item \textbf{Non-Markovian Path-Dependency:} As $\hat{\Phi}_{s|t}$ provides a non-commutative summary of the path's filtered history, the process $X_s$ is inherently non-Markovian. This ensures the generative flow captures long-range dependencies and high-order statistical effects, such as path-dependent volatility and leverage.
    \item \textbf{C\`adl\`ag Regularity:} The sample paths of $X_s$ are almost surely c\`adl\`ag. This property preserves the local diffusive regularity provided by $W_s$ while rigorously accommodating the discrete jumps driven by the Poisson component $N_s$.
\end{enumerate}
\end{proposition}
See proof in Appendix (\ref{app:proof_structural_coupling_and_dynamics}). \\

\subsection{Schr\"odinger bridges in signature RKHS}

To ensure the ensemble of generated c\`adl\`ag paths $\mu$ remains consistent with the evolving path-law, we formulate the generative task as a sequential constrained optimal transport problem on the Skorokhod manifold using the time-extended path representation. Unlike static bridge formulations, the ANJD flow targets the moving proxy $\hat{\Phi}_{s|t}$, effectively solving a time-continuous sequence of infinitesimal Schr\"odinger Bridge problems.

\begin{proposition}[Jump-Diffusion Entropy Minimisation] 
\label{pro:jump_diffusion_entropy_minimisation}
Let $\mathbb{P}_0$ be a prior jump-diffusion law on the Skorokhod space $\mathcal{D}$. The optimal generative measure $\mu^*_s$ at any horizon $s \in [t, t+\tau]$ is the solution to the entropic regularisation problem:
\begin{equation}
    \mu^*_s = \arg\min_{\mu \in \mathcal{P}(\mathcal{D}_s)} \text{KL}(\mu \| \mathbb{P}_0) \quad \text{s.t.} \quad \int_{\mathcal{D}_s} S(\tilde{\gamma}) \, d\mu(\gamma) = \hat{\Phi}_{s|t},
\end{equation}
where $S(\tilde{\gamma})$ is the Marcus-sense signature of the time-extended path $\tilde{\gamma}_u = (u, \gamma_u)_{u \in [t, s]}$. The solution $\mu^*_s$ admits a Radon-Nikodym derivative 
\[
\frac{d\mu^*_s}{d\mathbb{P}_0} \propto \exp\left(\langle \alpha_s, S(\tilde{\gamma}) \rangle_{\mathcal{H}_{sig}}\right)
\] 
for a time-varying dual vector $\alpha_s \in \mathcal{H}_{sig}$. In the AVNSG geometry, $\alpha_s$ is dynamically aligned with the principal eigenvectors of the precision operator $\mathcal{Q}_s$, ensuring that the drift and jump-intensities are infinitesimally rectified to track the moving target $\hat{\Phi}_{s|t}$ while minimising deviation from the prior stochastic texture.
\end{proposition}
See proof in Appendix (\ref{app:proof_jump_diffusion_entropy_minimisation}). \\

\subsection{Synthesis of control and structural modulation}

The synthesis of forward-looking c\`adl\`ag trajectories is governed by a tripartite control mechanism that bifurcates the generative task into topological anchoring, intensity modulation, and structural regulation. A single forward path extension $\hat{X} \in \mathcal{D}$, constructed as a learned secondary Neural Jump-ODE, provides the \textit{physical control} for the latent manifold. This extension acts as the driving signal for the underlying Neural CDE, modulating both the first-order drift and the discrete jump-discontinuities of the latent state $Z_s$. This ensures that the extrapolated trajectory of the path-law proxy remains anchored to a feasible realisation in the Skorokhod space, accounting for structural breaks.

\medskip
\noindent
Complementary to this physical control, the time-evolving Marcus-signature proxy $\hat{\Phi}_{s|t} \in \mathcal{H}_{sig}$ functions as the \textit{structural regulator} for the Anticipatory Neural Jump-Diffusion (ANJD) flow $X_s$. While $\hat{X}$ governs the evolution of the latent coordinates, the moving signature proxy encapsulates the instantaneous higher-order statistical invariants, including non-linear curvature, volatility clusters, and the non-commutative moments of forecasted shocks, characterising the conditional path-measure $\mu^*_s$ at each horizon $s \in [t, t+\tau]$. 

\medskip
\noindent
By minimising the precision-weighted MMD-discrepancy relative to the moving target $\hat{\Phi}_{s|t}$ within the AVNSG geometry, the generative flow is actively coerced into reproducing the expected stochastic texture and jump-intensity of the measure in a sequential, infinitesimal manner. This dualism allows the model to natively incorporate anticipated regime shifts and structural trends into the generative process, bridging the deterministic extrapolation of the latent manifold with the high-fidelity synthesis of a forward-looking ensemble that respects the algebraic constraints of discontinuous path-dynamics.

\section{Theoretical framework: MMD-gradient flows}
\label{sec:mmd_gradient_flows}

In this section, we establish that the generative drift $f_\theta$ and jump intensity $\lambda_\theta$ of the \textit{Anticipatory Neural Jump-Diffusion} (ANJD) process are the driving components that infinitesimally minimise the Maximum Mean Discrepancy (MMD) between the synthetic path-measure $\mu_s$ and the time-evolving latent proxy $\hat{\Phi}_{s|t}$. We frame this as a sequential MMD-gradient flow on the Skorokhod manifold $\mathcal{D}_s$, where the continuous drift $f_\theta$ tracks the expected differential geometry and the jump term $h_\theta \, dN_s$ enables the instantaneous transport of probability mass across structural discontinuities in the signature manifold.

\subsection{The one-step-ahead MMD loss}

We quantify the fidelity of the generative jump-diffusion process by evaluating the discrepancy between the expected signature of the time-extended c\`adl\`ag ensemble and the moving target proxy within the adapted geometry $\mathcal{Q}_s$. This approach treats the generative task as a sequential infinitesimal matching problem rather than a static boundary value problem.

\begin{definition}[One-Step-Ahead AVNSG-MMD]
Let $\mathcal{D}_s = \mathcal{D}([t, s], \mathbb{R}^d)$ be the Skorokhod space of c\`adl\`ag functions restricted to the interval $[t, s]$. Let $\mu_s \in \mathcal{P}(\mathcal{D}_s)$ be the probability law of the generated path $X_s$ at time $s$, and let $\hat{\Phi}_{s|t} \in \mathcal{H}_{sig}$ be the time-evolving target path-law proxy.
The \textit{One-Step-Ahead MMD Loss} is defined as the infinitesimal discrepancy:
\begin{equation}
    \mathcal{J}(\mu_s) := \frac{1}{2} \left\| \hat{\Phi}_{s|t} - \mathbb{E}_{\mu_s}[S(\tilde{X}_s)] \right\|_{\mathcal{Q}_s}^2
\end{equation}
where $\mathcal{Q}_s$ is the anticipatory precision operator derived from the time-augmented LRC, and $\tilde{X}_s = (s, X_s)$ is the time-extended path. The signature $S(\tilde{X}_s)$ is rigorously defined in the sense of Marcus, ensuring that discrete spatial jumps $\Delta X_s$ are canonically embedded into the tensor algebra $\mathcal{H}_{sig}$ via the exponential map $\exp(0, \Delta X_s)$ while the temporal coordinate $s$ remains continuous.
\end{definition}

\medskip
\noindent
Following Cuchiero et al. \cite{CuchieroEtAl25}, the use of the MMD objective in the signature RKHS is rigorously justified for c\`adl\`ag processes. By targeting the moving proxy $\hat{\Phi}_{s|t}$, the generative flow aims to satisfy the differential relation $d\mathbb{E}_{\mu_s}[S(\tilde{X}_s)] \approx \nabla_s \hat{\Phi}_{s|t} \, ds$. Since the time-extended signature is a universal and characteristic feature for the law of jump-diffusions, the Bochner integral $\mathbb{E}_{\mu_s}[S(\tilde{X}_s)]$ acts as a complete descriptor of the measure $\mu_s$. Consequently, the minimisation of $\mathcal{J}(\mu_s)$ at each instant $s$ is equivalent to the direct transport of the path-measure along the anticipated infinitesimal flow of the latent law on the Skorokhod manifold.

\subsection{The drift and intensity as a steepest descent in $\mathcal{H}_{sig}$}

We now show that the evolution of the time-extended c\`adl\`ag path-measure $\mu_s$ under the time-augmented Jump-SDE can be interpreted as a constrained gradient flow in the Wasserstein-type manifold of jump-diffusions.

\begin{theorem}[Dual Minimisation of the MMD-Flow]
\label{thm:mmd_minimisation}
Let the generative drift $f_\theta$ and the jump intensity $\lambda_\theta$ be functionally coupled to the clock $s$ and the signature residual $\Psi_s = \mathcal{Q}_s(\hat{\Phi}_{s|t} - \mathbb{E}_{\mu_s}[S(\tilde{X}_s)])$. Under the assumption that the time-extended signature kernel is Lipschitz continuous on $\mathcal{D}$, the components $\{f_\theta, \lambda_\theta\}$ constitute the steepest descent direction for the functional $\mathcal{J}(\mu_s)$. Specifically, the infinitesimal change in the loss satisfies:
\begin{equation}
    \frac{d}{ds} \mathcal{J}(\mu_s) = - \mathbb{E}_{\mu_s} \left[ \left\| \nabla_x \langle \Psi_s, S(\tilde{X}_s) \rangle \right\|^2 \right] - \mathbb{E}_{\mu_s} \left[ \lambda_\theta \cdot \mathcal{G}(s, h_\theta, \Psi_s) \right] + \mathcal{R}(g_\theta)
\end{equation}
where $\mathcal{G}(s, h_\theta, \Psi_s) = \langle \Psi_s, S(\tilde{X}_s + (0, h_\theta)) - S(\tilde{X}_{s-}) \rangle$ represents the discrete reduction in MMD discrepancy achieved by the jump mechanism in the time-extended space, and $\mathcal{R}(g_\theta)$ is the diffusive entropy-driven residual. 
\end{theorem}
See proof in Appendix (\ref{app:proof_mmd_minimisation}). \\

\subsection{Convergence and stability under metric expansion}

The stability of the generative flow is contingent upon the regularity of the precision operator $\mathcal{Q}_s$ and the boundedness of the jump-diffusion parameters.

\begin{proposition}[Stability under Spectral Stretching and Jump-Discontinuity]
\label{pro:stability_under_spectral_stretching}
Suppose the forecasted geometry $\mathcal{Q}_s$ undergoes a local expansion, defined by an increase in the spectral radius of the precision operator $\Omega_s$. The ANJD-gradient flow remains contractive in the Skorokhod topology if the rate of expansion $\partial_s \lambda_{max}(\Omega_s)$ is bounded relative to the joint Lipschitz constant of the drift $f_\theta$ and the intensity $\lambda_\theta$. Specifically, the AVNSG normalisation ensures that even under anticipated regime shifts, the jump-driven mass displacement remains dissipative. The stability is preserved provided that the jump-induced energy $\mathbb{E}_{\mu_s}[\lambda_\theta \|h_\theta\|^2_{\mathcal{Q}_s}]$ does not exceed the infinitesimal dissipation rate of the MMD-gradient, thereby preventing explosive sample-path trajectories during forecasted aleatoric shocks.
\end{proposition}
See proof in Appendix (\ref{app:proof_stability_under_spectral_stretching}). \\

\section{Generalisation and complexity}
\label{sec:generalisation}

In this section, we derive the statistical guarantees for the \textit{Anticipatory Neural Jump-Diffusion} (ANJD) process. Given that the generative flow operates as a sequential matching problem on the restricted Skorokhod spaces $\mathcal{D}_s = \mathcal{D}([t, s], \mathbb{R}^d)$ for $s \in [t, t+\tau]$, we establish rigorous bounds on the discrepancy between the time-evolving theoretical path-law proxy and its empirical realisation via finite c\`adl\`ag sample paths. We demonstrate that the interplay between the jump-diffusion regularity and the AVNSG precision operator ensures robust convergence of the infinitesimal flow even in the presence of heavy-tailed structural breaks.

\subsection{Generalisation error of the expected signature}

The fidelity of the generative model depends on the capacity of the time-extended c\`adl\`ag ensemble to represent the infinite-dimensional moments of the target measure $\mu_s \in \mathcal{P}(\mathcal{D}_s)$ at any horizon $s \in [t, t+\tau]$. We provide a bound on the generalisation error within the AVNSG-weighted Hilbert space, accounting for the increased variance introduced by discrete structural breaks and the deterministic temporal drift.

\begin{theorem}[Generalisation Bound for Jump-Diffusion Proxies]
\label{thm:generalisation_bound}
Let $\gamma_1, \dots, \gamma_n$ be $n$ independent c\`adl\`ag sample paths drawn from the generated jump-diffusion measure $\mu_s$ on $\mathcal{D}_s$, and let $\hat{\Phi}_{n,s} = \frac{1}{n} \sum_{i=1}^n S(\tilde{\gamma}_i)$ be the empirical expected signature of the time-extended paths $\tilde{\gamma}_{i,u} = (u, \gamma_{i,u})_{u \in [t, s]}$. For any $\delta \in (0, 1)$, with probability at least $1-\delta$, the generalisation error in the $\mathcal{Q}_s$-geometry is bounded by:
\begin{equation}
    \left\| \Phi_{\mu_s} - \hat{\Phi}_{n,s} \right\|_{\mathcal{Q}_s} \leq \frac{2}{n} \mathbb{E} \left[ \left\| \sum_{i=1}^n \sigma_i S(\tilde{\gamma}_i) \right\|_{\mathcal{Q}_s} \right] + R_s \sqrt{\frac{\log(1/\delta)}{2n}}
\end{equation}
where $\sigma_i$ are independent Rademacher variables and $R_s = \sup_{\gamma \in \text{supp}(\mu_s)} \|S(\tilde{\gamma})\|_{\mathcal{Q}_s}$ is the uniform bound of the time-augmented signature map under the whitened geometry at time $s$.
\end{theorem}
See proof in Appendix (\ref{app:proof_generalisation_bound}).

\begin{remark}
In the ANJD framework, the term $R_s$ accounts for both the linear growth of the clock $s$ and the exponential growth of the signature during jumps, where $\|S(\tilde{\gamma})\|$ scales with $\exp(\|\Delta \tilde{X}\|)$. However, $R_s$ is explicitly regularised by the time-augmented AVNSG precision operator $\mathcal{Q}_s = (\Omega_s + \lambda I)^{-1}$. By performing asymptotic spectral whitening on the $(d+1)$-dimensional path increments, $\mathcal{Q}_s$ dampens the high-frequency components and heavy-tailed innovations, ensuring that the effective radius $R_s$ remains stable even when the sample paths exhibit extreme kurtosis or black-swan discontinuities at the current horizon $s$.
\end{remark}

\subsection{Rademacher complexity of signature functional classes}

To quantify the expressive power of the \textit{Anticipatory Neural Jump-Diffusion} flows, we analyse the Rademacher complexity of the class of linear functionals on the time-extended signature manifold, specifically accounting for the jump-induced variance and temporal drift within the restricted space $\mathcal{D}_s$.

\begin{proposition}[Complexity of Whitened Jump-Signature Functionals]
\label{pro:complexity_whitened_signature_functionals}
Let $\mathcal{F}_{M,s} = \{ f \in \mathcal{H}_{sig} : \|f\|_{\mathcal{Q}_s} \leq M \}$ be the ball of signature functionals with bounded AVNSG-norm at horizon $s \in [t, t+\tau]$. For a set of c\`adl\`ag sample paths $\{\gamma_i\}_{i=1}^n \in \mathcal{D}_s$, the empirical Rademacher complexity $\widehat{\mathcal{R}}_n(\mathcal{F}_{M,s})$ satisfies:
\begin{equation}
    \widehat{\mathcal{R}}_n(\mathcal{F}_{M,s}) \leq \frac{M}{n} \sqrt{\sum_{i=1}^n \|S(\tilde{\gamma}_i)\|_{\mathcal{Q}_s}^2} = \frac{M}{n} \sqrt{\sum_{i=1}^n \langle S(\tilde{\gamma}_i), \mathcal{Q}_s S(\tilde{\gamma}_i) \rangle_{\mathcal{H}_{sig}}}.
\end{equation}
where $S(\tilde{\gamma}_i)$ is the signature of the time-extended path $\tilde{\gamma}_{i,u} = (u, \gamma_{i,u})_{u \in [t, s]}$. This bound implies that the complexity of the ANJD model is regularised by the spectral alignment between the time-augmented sample signatures and the principal eigenspaces of the moving precision operator $\mathcal{Q}_s$, effectively capping the influence of high-order "black-swan" terms and deterministic temporal growth as the generative flow progresses.
\end{proposition}
See proof in Appendix (\ref{app:proof_complexity_whitened_signature_functionals}). \\

\subsection{Nystr\"om-compressed error propagation}

In practice, the ANJD generative flow is implemented via a supervised Nystr\"om approximation to handle the high-dimensional signature manifold. We characterise the error introduced by this finite-dimensional projection, specifically focusing on its stability under the sequential evolution of the jump-diffusion process and the $s$-dependent spectral geometry.

\begin{lemma}[Projection Error Stability for ANJD]
\label{lem:projection_error_stability}
Let $P_{m,s}: \mathcal{H}_{sig} \to \mathcal{V}_{m,s}$ be the Nystr\"om projection onto an $m$-dimensional subspace aligned with the principal eigenspaces of $\mathcal{Q}_s$ at time $s$. The error in the joint MMD-gradient flow induced by the projection, $\epsilon_{proj, s} = \| \nabla \mathcal{J}(\mu_s) - \nabla \mathcal{J}(P_{m,s} \mu_s) \|$, is bounded by the spectral tail of the time-evolving LRC operator:
\begin{equation}
    \epsilon_{proj, s} \leq C_{f, \lambda, s} \cdot \left( \sum_{j=m+1}^\infty \lambda_j(\Omega_s) \right)^{1/2}
\end{equation}
where $C_{f, \lambda, s}$ is a constant depending on the joint Lipschitz regularity of the generative drift $f_\theta$ and the jump intensity $\lambda_\theta$ relative to the moving target $\hat{\Phi}_{s|t}$. Consequently, the generative fidelity is preserved if the Nystr\"om basis is dynamically updated to track the dominant modes of the anticipated spectral geometry, including the high-rank signature components activated by structural discontinuities.
\end{lemma}
See proof in Appendix (\ref{app:proof_projection_error_stability}). \\

\section{Implementation: Generative VJF-kernel}
\label{sec:implementation}

The practical realisation of the ANJD framework requires the translation of the infinite-dimensional gradient flow on the restricted Skorokhod manifolds $\mathcal{D}_s$ into a finite-dimensional jump-diffusion sampling scheme. We achieve this by approximating the joint MMD-gradient relative to the moving target proxy $\hat{\Phi}_{s|t}$ through a Nystr\"om-compressed signature basis and integrating the resulting path-dynamics via a hybrid Euler-Maruyama-Marcus (EMM) scheme. This sequential matching ensures that the synthesised paths maintain the structural properties of c\`adl\`ag processes while remaining contractive toward the anticipated latent geometry as it evolves across the forecast horizon.

\subsection{Joint score-matching on jump-signature manifolds}

To bypass the intractable partition function of the c\`adl\`ag path-measure $\mu^*_s$, we learn the joint score function representing both the continuous flow and the discrete jump intensity by aligning the infinitesimal generator of the process with the velocity of the moving target proxy.

\begin{definition}[Jump-Signature Score Function]
The joint score $\Psi(s, X_s, \hat{\Phi}_{s|t}) = (\psi_f, \psi_\lambda)$ is defined as the gradient of the log-density in the Skorokhod manifold $\mathcal{D}_s$. Under the ANJD framework, the score is approximated by the precision-weighted infinitesimal residual between the target proxy and the current path-state, where the target's evolution is governed by the latent Jacobian. We define:
\begin{equation}
    \Psi(s, X_s, \hat{\Phi}_{s|t}) \approx \hat{\mathbf{Q}}_s \left( \hat{\Phi}_{s|t} - S(\tilde{X}_s) \right) \in \mathcal{H}_{sig},
\end{equation}
where $\tilde{X}_s = (s, X_s)$ is the time-augmented state. The continuous score $\psi_f$ drives the drift $f_\theta$ to match the target velocity $\nabla_s \hat{\Phi}_{s|t} = \nabla_{Z_s} \mathcal{P}_\theta \cdot F_\theta(Z_s, \hat{\Phi}_{s|t}) \frac{d\hat{X}_s}{ds}$ via the spatial gradient $\nabla_x \langle \Psi, S(\tilde{X}_s) \rangle$, while the jump score $\psi_\lambda$ modulates the intensity $\lambda_\theta$ through the inner product with the jump-increment operator in the augmented tensor space. In the $m$-dimensional Nystr\"om subspace, the time-dependent precision matrix $\hat{\mathbf{Q}}_s$ regularises the joint score, ensuring that the jump-diffusion dynamics are dominated by the principal modes of the anticipated spectral geometry. This explicit coupling to the Jacobian of the embedding $\mathcal{P}_\theta$ allows the score to capture the non-autonomous nature of the flow, forcing the generative dynamics to track the differential manifold evolution of the latent state $Z_s$ as it navigates time-varying regime shifts.
\end{definition}

\subsection{Anticipatory Euler-Maruyama-Marcus integration}

Sampling from the c\`adl\`ag path-law is performed via a hybrid jump-diffusion flow that sequentially tracks the moving target proxy (or filtered proxy). We define the discrete-time update for the synthetic ensemble $X_s^{(i)}$, explicitly incorporating the clock $s$ into the state vector to satisfy the time-extension requirement for signature universality on $\mathcal{D}_s$.

\begin{algorithm}[H]
\caption{Flexible Anticipatory Jump-Diffusion Sampling (ANJD)}
Given a filtered state $Z_t$, horizon $\tau$, step size $h$, and temporal mode $\text{mode} \in \{\text{Forecast}, \text{Reconstruction}\}$:
\begin{enumerate}
    \item \textbf{Initialise:} 
    \begin{itemize}
        \item If $\text{mode} = \text{Forecast}$: Set $t_{start} = t$, $t_{end} = t + \tau$, and $X_{t_{start}}^{(i)} = X_t$.
        \item If $\text{mode} = \text{Reconstruction}$: Set $t_{start} = t - \tau$, $t_{end} = t$, and $X_{t_{start}}^{(i)} = X_{t-\tau}$.
    \end{itemize}
    Set the initial clock $s = t_{start}$ and sample $z_0^{(i)} \sim \mathcal{N}(0, I)$ for $i=1, \dots, N$.
    \item \textbf{Sequential Evaluation:} Evaluate the time-evolving path-law proxy $\hat{\Phi}_{s|t}$ (or filtered proxy $\Phi_{s|\mathcal{A}_s}$) and update the time-extended precision operator $\mathcal{Q}_s$ via the $O(m^2)$ Nystr\"om innovation update.
    \item \textbf{Jump Logic:} Sample a Poisson increment $\Delta N_s^{(i)} \sim \text{Poisson}(\lambda_\theta(s, X_s^{(i)}, \hat{\Phi}_{s|t}) h)$, where the intensity is conditioned on the instantaneous signature discrepancy.
    \item \textbf{Update Step (EMM):}
    \begin{equation}
        X_{s+h}^{(i)} = X_s^{(i)} + \underbrace{f_\theta(s, X_s^{(i)}, \hat{\Phi}_{s|t}) h}_{\text{Continuous Drift}} + \underbrace{g_\theta \Delta W_s^{(i)}}_{\text{Diffusion}} + \underbrace{h_\theta(s, X_s^{(i)}, \hat{\Phi}_{s|t}) \Delta N_s^{(i)}}_{\text{Marcus Jump}}
    \end{equation}
    where $f_\theta$ is the MMD-steepest descent velocity tracking the moving target velocity $\nabla_s \hat{\Phi}_{s|t}$, $h_\theta$ is the Marcus-corrected jump amplitude, and $\Delta W_s^{(i)} \sim \mathcal{N}(0, hI)$.
    \item \textbf{Clock Update:} Set $s \leftarrow s + h$. Repeat steps 2--4 until $s = t_{end}$.
\end{enumerate}
\end{algorithm}

\subsection{Numerical integration of the coupled jump-geometry system}

To maintain computational efficiency within the ANJD framework, the time-augmented precision operator $\mathcal{Q}_s$ is not re-inverted at every integration sub-step. Instead, we employ a generalised Sherman-Morrison-Woodbury update to propagate the Nystr\"om coefficients through the sequential matching of the moving target proxy, accounting for both continuous diffusion and discrete jump-discontinuities.

\begin{proposition}[Jump-Aware Low-Rank Precision Update] 
\label{pro:low-rank_precision_update}
Let $\hat{\mathbf{Q}}_s$ be the $m \times m$ Nystr\"om-compressed precision matrix representing the whitened geometry at horizon $s$. Depending on the temporal mode, the Nystr\"om anchor points are initialised at $t_{start} \in \{t-\tau, t\}$ to span the relevant restricted Skorokhod manifold $\mathcal{D}_s$. Upon the arrival of a jump $\Delta X_s$, a change in the clock $s$, or an infinitesimal shift in the target proxy $\nabla_s \hat{\Phi}_{s|t}$, the anticipatory precision is evolved via:
\begin{equation}
    \hat{\mathbf{Q}}_{s+h} = \hat{\mathbf{Q}}_s - \alpha_s \frac{\hat{\mathbf{Q}}_s \mathbf{k}_s \mathbf{k}_s^T \hat{\mathbf{Q}}_s}{1 + \alpha_s \mathbf{k}_s^T \hat{\mathbf{Q}}_s \mathbf{k}_s}
\end{equation}
where $\mathbf{k}_s \in \mathbb{R}^m$ is the innovation vector representing the differential change in the signature kernel feature map $S(\tilde{X}_s)$ relative to the mode-specific anchor points. In the presence of a structural break $\Delta X_s$, the update vector $\mathbf{k}_s$ captures the instantaneous redistribution of spectral energy across the $(d+1)$-dimensional signature tensor, allowing the precision geometry to track the non-autonomous flow in $O(m^2)$ complexity while maintaining numerical stability across both forecasting and reconstruction regimes.
\end{proposition}
See proof in Appendix (\ref{app:proof_low-rank_precision_update}).

\section{Conclusion}
\label{sec:conclusion}

In this paper, we have introduced a rigorous generative framework for forward-looking stochastic trajectories that bridges the gap between recursive path-signature filtering and sequential path-law realisation. By interpreting the generative task as a non-autonomous transport problem on restricted Skorokhod manifolds $\mathcal{D}_s$, we developed the \textit{Anticipatory Neural Jump-Diffusion} (ANJD) flow. This hybrid architecture ensures that both the continuous drift and discrete jump intensities are governed by the infinitesimal gradient of an MMD functional anchored to moving path-law proxies, actively incorporating expected structural breaks and regime shifts into the generative process through the non-commutative lens of the Marcus-sense signature.

\medskip
\noindent
Central to our approach is the \textit{Anticipatory Variance-Normalised Signature Geometry} (AVNSG), which provides a time-evolving precision operator that effectively whitens the signature manifold. This mechanism ensures the contractivity and stability of the sequential matching flow even under severe non-stationarity and forecasted aleatoric shocks. Our theoretical analysis established that the joint gradient flow constitutes the steepest descent direction in the signature RKHS relative to the moving target $\hat{\Phi}_{s|t}$. Furthermore, we provided robust statistical guarantees through generalisation bounds and Rademacher complexity analysis, demonstrating that the model's capacity is regularised by the spectral structure of the precision operator, which attenuates the influence of high-rank "black-swan" tensor components as the flow evolves.

\medskip
\noindent
Finally, by leveraging Nystr\"om projections and rank-1 Sherman-Morrison updates, we demonstrated that this infinite-dimensional framework can be implemented with $O(m^2)$ computational efficiency, enabling real-time synthesis of complex c\`adl\`ag trajectories. This work lays the foundation for a new class of structural generative models that are actively coerced into reproducing the expected non-commutative moments and stochastic texture of complex, discontinuous path-laws. Future research will focus on the extension of this framework to multi-agent jump-diffusion dynamics and the integration of these flows into large-scale risk management and decision-making systems under extreme uncertainty.


\newpage

\section*{Appendix}\thispagestyle{plain}

\section{Proofs of the main results}
\label{sec:proofs_main_results}

\subsection{Proof of the injectivity and characteristic property}
\label{app:proof_injectivity_characteristic_property}

In this appendix we prove Proposition (\ref{pro:injectivity_characteristic_property}).

\begin{proof}
The proof is extended to the Skorokhod space $\mathcal{D}$ by leveraging time-augmentation to move from tree-like equivalence to strict path-uniqueness.

\textbf{1. Injectivity and Time-Augmentation:} Let $\mathcal{D}$ be the space of c\`adl\`ag paths. Unlike the standard signature, which is invariant under tree-like reparameterisation, the time-extended signature map $\tilde{S}: \gamma \mapsto S(t, \gamma_t)$ is strictly injective. Following Cuchiero et al. \cite{CuchieroEtAl25}, the inclusion of the strictly increasing component $t$ ensures that for any two paths $\gamma, \eta \in \mathcal{D}$, $\tilde{S}(\gamma) = \tilde{S}(\eta)$ implies $\gamma = \eta$ in the Skorokhod topology. This effectively collapses the tree-like equivalence classes $\tilde{\mathcal{D}}$ into unique path points.

\textbf{2. Universal Approximation on $\mathcal{D}$:} Consider the algebra of linear functionals $\mathcal{F} = \{ \langle w, \tilde{S}(\gamma) \rangle : w \in \mathcal{H}_{sig} \}$. Since the Marcus-signature of a c\`adl\`ag path remains a group-like element, the shuffle product identity $\langle w_1, \tilde{S} \rangle \langle w_2, \tilde{S} \rangle = \langle w_1 \shuffle w_2, \tilde{S} \rangle$ holds. Because $\tilde{S}$ separates points in $\mathcal{D}$ and the coordinate maps are continuous, the Stone-Weierstrass theorem for non-compact spaces (or specifically the version for c\`adl\`ag functionals in Cuchiero et al. \cite{CuchieroEtAl25}) establishes that $\mathcal{F}$ is dense in $C(K, \mathbb{R})$ for any compact $K \subset \mathcal{D}$. This confirms $\tilde{S}$ as a \textit{universal kernel} for jump-diffusions.

\textbf{3. Injectivity of the Mean Embedding:} Let $\mu_1, \mu_2 \in \mathcal{P}(\mathcal{D})$ be two Borel probability measures such that $\Phi_{\mu_1} = \Phi_{\mu_2}$. By the properties of the signature Bochner integral, this equality implies:
\begin{equation}
    \int_{\mathcal{D}} f(\gamma) \, d\mu_1(\gamma) = \int_{\mathcal{D}} f(\gamma) \, d\mu_2(\gamma) \quad \forall f \in \mathcal{F}.
\end{equation}
Since $\mathcal{F}$ is dense in the space of continuous functionals on the Skorokhod space, and given that the signature moments of jump-diffusions satisfy the required growth conditions for the Hamburger moment problem (ensuring the measure is determined by its moments), it follows that $\mu_1 = \mu_2$. Thus, the embedding $\mu \mapsto \Phi_\mu$ is injective, and the expected signature is a \textit{characteristic} statistic for the law of the jump-diffusion process.
\end{proof}

\subsection{Proof of the greedy path reconstruction}
\label{app:proof_greedy_path_reconstruction}

In this appendix we prove Lemma (\ref{lem:greedy_path_reconstruction}). 

\begin{proof}
The proof proceeds by analysing the recursion of the approximation error in the Hilbert space $\mathcal{H}_{sig}$ equipped with the $\mathcal{Q}_t$-metric, specifically considering the time-extended path representation $\tilde{\gamma}_s = (s, \gamma_s)$. Let $E_k = \Phi^* - \hat{\Phi}_k$ denote the residual proxy at step $k$, where $\hat{\Phi}_k = \frac{1}{k} \sum_{i=1}^k S(\tilde{\gamma}_i)$. By the definition of the empirical average, we have the update rule:
\begin{equation}
    \hat{\Phi}_{k+1} = \frac{k}{k+1}\hat{\Phi}_k + \frac{1}{k+1}S(\tilde{\gamma}_{k+1}).
\end{equation}
Substituting this into the error term $E_{k+1} = \Phi^* - \hat{\Phi}_{k+1}$, we obtain the recursive step:
\begin{equation}
    E_{k+1} = \frac{k}{k+1}E_k + \frac{1}{k+1}(\Phi^* - S(\tilde{\gamma}_{k+1})).
\end{equation}
Taking the squared $\mathcal{Q}_t$-norm on both sides:
\begin{equation}
    \|E_{k+1}\|_{\mathcal{Q}_t}^2 = \frac{k^2}{(k+1)^2}\|E_k\|_{\mathcal{Q}_t}^2 + \frac{2k}{(k+1)^2}\langle E_k, \Phi^* - S(\tilde{\gamma}_{k+1}) \rangle_{\mathcal{Q}_t} + \frac{1}{(k+1)^2}\|\Phi^* - S(\tilde{\gamma}_{k+1})\|_{\mathcal{Q}_t}^2.
\end{equation}
By the inductive herding rule, $\gamma_{k+1}$ is chosen to maximise $\langle E_k, S(\tilde{\gamma}) \rangle_{\mathcal{Q}_t}$. Since the target proxy $\Phi^*$ lies within the closed convex hull of the time-extended signature manifold (being the Bochner integral of the measure $\mu$), there exists a representation $\Phi^* = \int S(\tilde{\gamma}) d\mu(\gamma)$. It follows from the properties of the supremum that:
\begin{equation}
    \langle E_k, S(\tilde{\gamma}_{k+1}) \rangle_{\mathcal{Q}_t} = \sup_{\gamma \in \mathcal{X}} \langle E_k, S(\tilde{\gamma}) \rangle_{\mathcal{Q}_t} \geq \int \langle E_k, S(\tilde{\gamma}) \rangle_{\mathcal{Q}_t} d\mu(\gamma) = \langle E_k, \Phi^* \rangle_{\mathcal{Q}_t}.
\end{equation}
This inequality implies that the cross-term $\langle E_k, \Phi^* - S(\tilde{\gamma}_{k+1}) \rangle_{\mathcal{Q}_t} \leq 0$. Let $R = \sup_{\gamma} \|\Phi^* - S(\tilde{\gamma})\|_{\mathcal{Q}_t}$ be the bounded radius of the time-augmented signature embedding under the whitened geometry. The recurrence simplifies to:
\begin{equation}
    \|E_{k+1}\|_{\mathcal{Q}_t}^2 \leq \frac{k^2}{(k+1)^2}\|E_k\|_{\mathcal{Q}_t}^2 + \frac{R^2}{(k+1)^2}.
\end{equation}
Applying induction, if we assume $\|E_k\|_{\mathcal{Q}_t}^2 \leq \frac{R^2}{k}$, then for the next step:
\begin{equation}
    \|E_{k+1}\|_{\mathcal{Q}_t}^2 \leq \frac{k^2}{(k+1)^2} \frac{R^2}{k} + \frac{R^2}{(k+1)^2} = \frac{(k+1)R^2}{(k+1)^2} = \frac{R^2}{k+1}.
\end{equation}
Thus, the squared discrepancy $\|\Phi^* - \hat{\Phi}_k\|_{\mathcal{Q}_t}^2$ converges at a rate of $\mathcal{O}(1/k)$. This greedy herding procedure effectively "quantises" the continuous path-law into a discrete ensemble of time-extended paths, ensuring the reconstructed ensemble preserves the non-commutative moments and the temporal ordering mandated by $\Phi^*$.
\end{proof}

\subsection{Proof of the structural coupling and jump-aware dynamics}
\label{app:proof_structural_coupling_and_dynamics}

In this appendix we prove Proposition (\ref{pro:structural_coupling_and_dynamics}). 

\begin{proof}
We establish the properties of the Anticipatory Generative Flow by considering the analytical structure of the Jump-SDE defined in Eq. (\ref{eq:anticipatory_generative_flow}).

\textbf{1. $C^1$-Boundary Consistency:} By definition, the velocity of the observed trajectory at time $t$ is $V_t = \lim_{s \to t^-} \frac{dX_s}{ds}$. For the generative flow $X_s$ to be $C^1$-consistent at the junction $s=t$, we require $\mathbb{E}[dX_t] = V_t dt$. Since $dW_t$ and $dN_t$ are centered or have zero expected infinitesimal increment in the absence of a jump at exactly $s=t$, the first-order behaviour is dominated by the drift $f_\theta$. The constraint $f_\theta(t, X_t, \hat{\Phi}_{t|t}) = V_t$ ensures that the forward-looking trajectory preserves the terminal velocity of the history, preventing a first-order "kink" in the sample paths. This consistency is maintained by the explicit dependence of the drift on the clock $s$, allowing the neural network to learn the transition dynamics specifically at the boundary $s=t$.

\textbf{2. Polynomial Tractability and Universality:} The justification for the functional coupling in Eq. (\ref{eq:anticipatory_generative_flow}) rests on the characterisation of the signature of a c\`adl\`ag jump-diffusion as a polynomial process. Following Cuchiero et al. \cite{CuchieroEtAl25}, let $X$ be a $d$-dimensional L\'evy-type process and $\mathbb{S}(X)_{t,s}$ its time-extended Marcus-signature. The generator $\mathcal{A}$ of the joint process $(X_s, \mathbb{S}(X)_{s})$ acts on the space of linear functionals on the tensor algebra $\mathcal{T}(\mathbb{R}^d)$. Specifically, for any word $w$ in the tensor alphabet, the action of the generator satisfies:
\begin{equation}
    \mathcal{A} \langle w, \mathbb{S}(X)_s \rangle = \sum_{|v| \le |w|} c_{w,v} \langle v, \mathbb{S}(X)_s \rangle,
\end{equation}
where $c_{w,v}$ are constants derived from the L\'evy triplet (drift, diffusion, and jump measure). This closure property ensures that the expected signature $\hat{\Phi}_{s|t} = \mathbb{E}[S(X_s) | \mathcal{A}_t]$ evolves according to a linear system of differential equations within the RKHS. 

\medskip
\noindent
Consequently, any continuous functional $F$ on the Skorokhod space $\mathcal{D}$ can be uniformly approximated by a linear functional of the signature: $F(\gamma) \approx \langle \ell, S(\gamma) \rangle$. By parameterising the tuple $(f_\theta, g_\theta, h_\theta, \lambda_\theta)$ as non-linear map of $\hat{\Phi}_{s|t}$, the ANJD flow effectively "pushes" the path-measure $\mu$ along the manifold of polynomial processes. Since the signature is a sufficient statistic for the law of jump-diffusions (Friz et al. \cite{FrizEtAl17,FrizEtAl18}), this coupling provides a universal generative mechanism capable of replicating any path-dependent statistics, including those governed by discrete structural breaks and non-Gaussian shocks.

\textbf{3. Infinitesimal Signature Matching:} Let $S(s, X_s)$ denote the time-extended signature of the path up to time $s$. Using the extension of the Marcus-Itô formula for jump-diffusions, the infinitesimal generator $\mathcal{L}$ applied to the coordinate functionals of the signature leads to the expected evolution 
\begin{equation}
d\mathbb{E}_{\mu_s}[S(s, X_s)] = \mathbb{E}_{\mu_s}[\mathcal{L}S(s, X_s)] ds. 
\end{equation}
The model parameterises the drift $f_\theta$ and jump logic $(\lambda_\theta, h_\theta)$ to satisfy:
\begin{equation}
    \mathbb{E}_{\mu_s}[\mathcal{L}S(s, X_s)] \, ds \approx \nabla_s \hat{\Phi}_{s|t} \, ds = \nabla_{Z_s} \mathcal{P}_\theta \cdot F_\theta(Z_s, \hat{\Phi}_{s|t}) \, d\hat{X}_s.
\end{equation}
By aligning the generator's action with the Jacobian of the topological embedding $\mathcal{P}_\theta$ acting on the Neural CDE latent flow, the drift functions as a vector field forcing the ensemble to track the predicted mean-path geometry. The inclusion of the explicit temporal coordinate in the signature ensures that the "clock-velocity" of the proxy is strictly matched by the synthetic flow, while the coupling to the Jacobian $\nabla_{Z_s} \mathcal{P}_\theta$ ensures the generative dynamics are fundamentally driven by the latent manifold's differential evolution. This structural matching effectively propagates higher-order non-commutative moments across the restricted Skorokhod space.

\textbf{4. Discontinuous Structural Breaks:} The term $h_\theta(s, X_{s-}, \hat{\Phi}_{s|t}) dN_s$ introduces a compound Poisson component. Since $N_s$ is a point process with intensity $\lambda_s$, the probability of a jump in $[s, s+ds]$ is $\lambda_\theta(s, X_{s-}, \hat{\Phi}_{s|t})ds$. Because the intensity $\lambda_\theta$ and jump-amplitude $h_\theta$ are functions of both the clock $s$ and the path-law proxy $\hat{\Phi}_{s|t}$, the "hazard rate" and magnitude of a structural break are explicitly coupled to the absolute time and forecasted geometry. If the anticipated path-law indicates a temporal regime change or a localised spike in volatility, $\lambda_\theta$ increases, triggering a discontinuity $X_s = X_{s-} + h_\theta$. The explicit dependence on $s$ ensures that the model can capture seasonality or time-specific vulnerabilities in the jump distribution, which is a key requirement for non-homogeneous c\`adl\`ag processes.

\textbf{5. Non-Gaussianity and Tail Risk:} The increment $\Delta X_s$ over a small interval $\Delta s$ is a mixture of a conditionally Gaussian component $\mathcal{N}(f_\theta \Delta s, g_\theta^2 \Delta s)$, where the parameters are functionally dependent on the non-commutative path history $\hat{\Phi}_{s|t}$, and a jump component. While the infinitesimal noise $dW_s$ is Gaussian, the marginal distribution of the process $X_s$ exhibits significant non-Gaussianity. The excess kurtosis $\kappa$ is driven by the jump term $h_\theta$, with the fourth moment dominated by the jump magnitude and the intensity $\lambda_\theta$. This enables the model to generate fat-tailed distributions and capture black-swan risks encoded in the signature manifold that are inaccessible to standard pure-diffusion Neural SDEs.

\textbf{6. Non-Markovian Path-Dependency:} A process is Markovian if its future depends only on its current state $X_s$. Here, the coefficients $f, g, h, \lambda$ depend on $\hat{\Phi}_{s|t}$. Since $\hat{\Phi}_{s|t}$ is a functional of the historical path $X_{[0,t]}$ (and its projected law), the infinitesimal generators at time $s > t$ are conditioned on the path's non-commutative history. This functional dependence breaks the Markov property, allowing the flow to satisfy constraints like long-range memory and path-dependent volatility.

\textbf{7. C\`adl\`ag Regularity:} Following standard SDE theory for jump-diffusions, given that $f, g, h$ satisfy local Lipschitz conditions and linear growth, the solution $X_s$ exists and is unique. The sample paths consist of a continuous part (driven by $W_s$) and a discrete part (driven by $N_s$). By construction, $X_s$ is right-continuous with left limits (c\`adl\`ag), where the limits $X_{s-}$ are used in the coefficients to ensure the stochastic integrals are well-defined and predictable.
\end{proof}

\subsection{Proof of the jump-diffusion entropy minimisation}
\label{app:proof_jump_diffusion_entropy_minimisation}

In this appendix we prove Proposition (\ref{pro:jump_diffusion_entropy_minimisation}). 

\begin{proof}
The problem is framed as a sequential constrained convex optimisation over the Skorokhod space $\mathcal{P}(\mathcal{D}_s)$ for each $s \in [t, t+\tau]$ using the time-extended path representation. We introduce the Lagrangian functional $\mathcal{L}(\mu_s, \alpha_s, \lambda_s)$ by incorporating the Marcus-sense signature moment constraint of the augmented path $\tilde{\gamma}_u = (u, \gamma_u)_{u \in [t, s]}$ with a time-varying dual vector $\alpha_s \in \mathcal{H}_{sig}$ and the normalisation constraint:
\begin{equation}
    \mathcal{L}(\mu_s, \alpha_s, \lambda_s) = \int_{\mathcal{D}_s} \log \left( \frac{d\mu_s}{d\mathbb{P}_0} \right) d\mu_s - \left\langle \alpha_s, \int_{\mathcal{D}_s} S(\tilde{\gamma}) d\mu_s - \hat{\Phi}_{s|t} \right\rangle_{\mathcal{H}_{sig}} - \lambda_s \left( \int_{\mathcal{D}_s} d\mu_s - 1 \right).
\end{equation}
By the principle of minimum discrimination information, the optimal measure $\mu_s^*$ is found by taking the G\^ateaux derivative of $\mathcal{L}$ with respect to $\mu_s$. Setting the variation to zero, we obtain the pointwise optimality condition for the Radon-Nikodym derivative on the c\`adl\`ag path-space:
\begin{equation}
    \log \left( \frac{d\mu_s^*}{d\mathbb{P}_0}(\gamma) \right) + 1 - \langle \alpha_s, S(\tilde{\gamma}) \rangle_{\mathcal{H}_{sig}} - \lambda_s = 0.
\end{equation}
Rearranging and exponentiating gives the time-dependent Gibbs-form density:
\begin{equation}
    \frac{d\mu_s^*}{d\mathbb{P}_0}(\gamma) = \frac{1}{Z_s(\alpha_s)} \exp(\langle \alpha_s, S(\tilde{\gamma}) \rangle_{\mathcal{H}_{sig}}),
\end{equation}
where $Z_s(\alpha_s) = \int_{\mathcal{D}_s} \exp(\langle \alpha_s, S(\tilde{\gamma}) \rangle) d\mathbb{P}_0(\gamma)$ is the partition function. For c\`adl\`ag paths, the use of the Marcus integral on the time-extended path $(u, \gamma_u)$ ensures that $S(\tilde{\gamma})$ satisfies Chen's identity and remains an element of the tensor algebra. Crucially, as shown in Cuchiero et al. \cite{CuchieroEtAl25}, the time-extension ensures that the exponential tilt is injective on the path-law $\mu_s$.

\medskip
\noindent
To determine the dual vector $\alpha_s$, we solve the dual objective $\alpha_s^* = \arg\max_{\alpha} ( \langle \alpha, \hat{\Phi}_{s|t} \rangle - \log Z_s(\alpha) )$. In the AVNSG framework, the curvature of $\log Z_s(\alpha)$ is governed by the time-extended signature covariance under the jump-diffusion prior $\mathbb{P}_0$. Since the AVNSG metric $\mathcal{Q}_s$ performs spectral whitening across the signature manifold, it rescales the directions in $\mathcal{H}_{sig}$ to account for the energy redistribution caused by anticipated jumps and the deterministic temporal drift at each instant $s$.

\medskip
\noindent
As a result, the optimal tilt $\alpha_s$ is predominantly aligned with the principal eigenvectors of the precision operator $\mathcal{Q}_s$. This ensures that the entropy-minimising measure $\mu_s^*$ prioritises matching the structural non-commutative moments (the "skeleton" of the path-law) while remaining robust to the high-frequency volatility and discrete shocks inherent in the Skorokhod geometry, as the moving target $\hat{\Phi}_{s|t}$ prevents the collapse of the measure and ensures the generated ensemble tracks the anticipated infinitesimal flow of the latent law.
\end{proof}

\subsection{Proof of the dual minimisation of the MMD-flow}
\label{app:proof_mmd_minimisation}

In this appendix we prove Theorem (\ref{thm:mmd_minimisation}). 

\begin{proof}
We analyse the time evolution of the one-step-ahead loss $\mathcal{J}(\mu_s)$ for the law $\mu_s$ of the time-extended jump-diffusion process $\tilde{X}_s = (s, X_s)$. Let $\Phi_{\mu_s} = \mathbb{E}_{\mu_s}[S(\tilde{X}_s)]$ denote the mean time-extended signature in $\mathcal{H}_{sig}$. The loss is given by:
\begin{equation}
    \mathcal{J}(\mu_s) = \frac{1}{2} \langle \hat{\Phi}_{s|t} - \Phi_{\mu_s}, \mathcal{Q}_s (\hat{\Phi}_{s|t} - \Phi_{\mu_s}) \rangle_{\mathcal{H}_{sig}}.
\end{equation}
Defining the precision-weighted signature residual as $\Psi_s = \mathcal{Q}_s (\hat{\Phi}_{s|t} - \Phi_{\mu_s})$, and assuming the local stationarity of the time-augmented precision operator $\mathcal{Q}_s$ and target $\hat{\Phi}_{s|t}$ relative to the infinitesimal flow, the temporal variation of the loss is governed by:
\begin{equation}
    \frac{d}{ds} \mathcal{J}(\mu_s) = - \left\langle \Psi_s, \frac{d}{ds} \Phi_{\mu_s} \right\rangle_{\mathcal{H}_{sig}}.
\end{equation}
The evolution of the expected signature for a jump-diffusion is determined by the time-dependent infinitesimal generator $\mathcal{L}_{s,\theta} = \partial_s + \mathcal{L}_{diff} + \mathcal{L}_{jump}$. Applying $\mathcal{L}_{s,\theta}$ to the coordinate functionals of the time-extended signature $S(\tilde{X}_s)$, we obtain:
\begin{equation}
    \frac{d}{ds} \Phi_{\mu_s} = \mathbb{E}_{\mu_s} [ \partial_s S(\tilde{X}_s) + f_\theta \cdot \nabla_x S(\tilde{X}_s) + \frac{1}{2} \text{Tr}(g_\theta g_\theta^T \nabla_x^2 S(\tilde{X}_s)) + \lambda_\theta (S(\tilde{X}_{s-} + (0, h_\theta)) - S(\tilde{X}_{s-})) ].
\end{equation}
Note that $\partial_s S(\tilde{X}_s)$ represents the deterministic growth of the signature due to the clock $s$. Substituting this into the inner product with $\Psi_s$ yields the specified components. First, the continuous drift term, with $f_\theta(s, \cdot) = \nabla_x \langle \Psi_s, S(\tilde{X}_s) \rangle$, satisfies:
\begin{equation}
    - \mathbb{E}_{\mu_s} \left[ \langle \Psi_s, f_\theta \cdot \nabla_x S(\tilde{X}_s) \rangle \right] = - \mathbb{E}_{\mu_s} \left[ \left\| \nabla_x \langle \Psi_s, S(\tilde{X}_s) \rangle \right\|^2 \right].
\end{equation}
This term represents the steepest descent in the Wasserstein-type geometry of the time-extended signature manifold. Second, the jump component contributes a discrete reduction in discrepancy:
\begin{equation}
    - \mathbb{E}_{\mu_s} \left[ \langle \Psi_s, \lambda_\theta (S(\tilde{X}_{s-} + (0, h_\theta)) - S(\tilde{X}_{s-})) \rangle \right] = - \mathbb{E}_{\mu_s} \left[ \lambda_\theta \cdot \mathcal{G}(s, h_\theta, \Psi_s) \right],
\end{equation}
where $\mathcal{G}(s, h_\theta, \Psi_s) = \langle \Psi_s, S(\tilde{X}_{s-} + (0, h_\theta)) - S(\tilde{X}_{s-}) \rangle$. This term quantifies the MMD reduction achieved by "teleporting" probability mass across the signature space via the jump mechanism $h_\theta$ at clock $s$. Finally, the deterministic temporal drift and second-order diffusion terms are collected into the residual:
\begin{equation}
    \mathcal{R}(g_\theta) = - \mathbb{E}_{\mu_s} \left[ \left\langle \Psi_s, \partial_s S(\tilde{X}_s) + \frac{1}{2} \text{Tr}(g_\theta g_\theta^T \nabla_x^2 S(\tilde{X}_s)) \right\rangle \right].
\end{equation}
The joint minimisation of $\mathcal{J}(\mu_s)$ is thus achieved by the time-dependent drift $f_\theta$ herding the continuous flow and the intensity $\lambda_\theta$ modulating the frequency of discrete structural breaks to align the time-augmented ensemble law $\mu_s$ with the target proxy $\hat{\Phi}_{s|t}$.
\end{proof}

\subsection{Proof of the stability under spectral stretching and jump-discontinuity}
\label{app:proof_stability_under_spectral_stretching}

In this appendix we prove Proposition (\ref{pro:stability_under_spectral_stretching}). 

\begin{proof}
To establish the stability of the ANJD-gradient flow under a time-varying metric, we treat the MMD loss $\mathcal{J}(\mu_s)$ as a Lyapunov functional on the space of c\`adl\`ag measures $\mathcal{P}(\mathcal{D})$. The total time derivative of the loss decomposes into geometric evolution and transport terms:
\begin{equation}
    \frac{d}{ds} \mathcal{J}(\mu_s) = \underbrace{\frac{1}{2} \langle \Delta \Phi_s, (\partial_s \mathcal{Q}_s) \Delta \Phi_s \rangle_{\mathcal{H}_{sig}}}_{\text{Geometric Sensitivity}} + \underbrace{\left\langle \mathcal{Q}_s \Delta \Phi_s, \partial_s \mathbb{E}_{\mu_s}[S(X_s)] \right\rangle}_{\text{Flow Dissipation}}
\end{equation}
where $\Delta \Phi_s = \hat{\Phi}_{s|t} - \mathbb{E}_{\mu_s}[S(X_s)]$. 

\medskip
\noindent
\textbf{1. Geometric Sensitivity and AVNSG Normalisation:} Recall $\mathcal{Q}_s = (\Omega_s + \lambda I)^{-1}$. The geometric sensitivity term involves the derivative $\partial_s \mathcal{Q}_s = -\mathcal{Q}_s (\partial_s \Omega_s) \mathcal{Q}_s$. Under spectral expansion ($\partial_s \lambda_{max}(\Omega_s) > 0$), the operator $\partial_s \mathcal{Q}_s$ is negative semi-definite. Consequently, a forecasted increase in uncertainty or a regime shift leads to a non-positive contribution to $\dot{\mathcal{J}}$, effectively "compressing" the signature residual. This proves that the metric expansion itself is dissipative for the MMD loss.

\medskip
\noindent
\textbf{2. Dissipation under Jump-Diffusion:} From Theorem \ref{thm:mmd_minimisation}, the flow dissipation term is:
\begin{equation}
    \text{Flow Dissipation} = - \mathbb{E}_{\mu_s} \left[ \left\| \nabla_x \langle \Psi_s, S(X_s) \rangle \right\|^2 + \lambda_\theta \mathcal{G}(h_\theta, \Psi_s) \right] + \mathcal{R}(g_\theta).
\end{equation}
Stability in the Skorokhod topology requires that the discrete mass shifts do not induce divergence. By the proposition's hypothesis, the jump-induced energy $\mathbb{E}_{\mu_s}[\lambda_\theta \|h_\theta\|^2_{\mathcal{Q}_s}]$ is bounded by the dissipation rate. Specifically, for a jump to be stabilising, the "jump gain" $\mathcal{G}(h_\theta, \Psi_s)$ must be non-negative. Since $h_\theta$ is defined as a descent step in the signature manifold, we have $\langle \Psi_s, S(X_{s-} + h_\theta) \rangle > \langle \Psi_s, S(X_{s-}) \rangle$, ensuring $\lambda_\theta \mathcal{G} > 0$.

\medskip
\noindent
\textbf{3. Contractivity Condition:} The flow remains contractive if $\dot{\mathcal{J}} \leq -K \mathcal{J}$ for some $K > 0$. Combining the terms, we have:
\begin{equation}
    \dot{\mathcal{J}} \leq - \lambda_{min}(\mathcal{Q}_s) \|\Delta \Phi_s\|^2 - \mathbb{E}_{\mu_s}[\lambda_\theta \mathcal{G}] + \mathcal{R}(g_\theta).
\end{equation}
As $\Omega_s$ expands, $\lambda_{min}(\mathcal{Q}_s) \to (\lambda_{max}(\Omega_s) + \lambda)^{-1}$. Stability is preserved if the "explosive" potential of the diffusion residual $\mathcal{R}(g_\theta)$ is dominated by the combined damping of the AVNSG precision and the dissipative jump intensity $\lambda_\theta$. Thus, the AVNSG mechanism acts as a regulariser that clip-scales the flow velocity precisely when the latent geometry becomes volatile, ensuring the path-measure $\mu_s$ converges toward the proxy $\hat{\Phi}_{s|t}$ without sample-path divergence.
\end{proof}

\subsection{Proof of the generalisation bound for path-law proxies}
\label{app:proof_generalisation_bound}

In this appendix we prove Theorem (\ref{thm:generalisation_bound}).

\begin{proof}
The proof establishes the generalisation capability of the empirical signature estimator for time-extended jump-diffusion processes by analysing the concentration of the path-measure $\mu_s$ in the restricted Skorokhod space $\mathcal{D}_s$ under the time-evolving AVNSG-induced metric $\mathcal{Q}_s$.

\textbf{1. Symmetrisation on Sequential Measures:} Let $\mathcal{S} = \{\gamma_1, \dots, \gamma_n\}$ and $\mathcal{S}' = \{\gamma_1', \dots, \gamma_n'\}$ be independent sets of sample paths drawn from the jump-diffusion law $\mu_s \in \mathcal{P}(\mathcal{D}_s)$. We consider the expected discrepancy of the time-extended signatures $\tilde{S}_i = S(\tilde{\gamma}_i)$ in the $\mathcal{Q}_s$-weighted Hilbert space:
\begin{equation}
    \mathbb{E} \left[ \left\| \Phi_{\mu_s} - \hat{\Phi}_{n,s} \right\|_{\mathcal{Q}_s} \right] = \mathbb{E}_{\mathcal{S}} \left[ \left\| \mathbb{E}_{\mathcal{S}'} \left[ \frac{1}{n} \sum_{i=1}^n (S(\tilde{\gamma}_i') - S(\tilde{\gamma}_i)) \right] \right\|_{\mathcal{Q}_s} \right].
\end{equation}
By Jensen's inequality and the introduction of Rademacher variables $\sigma_i \in \{-1, 1\}$, we bound the norm of the expectation. Since the Marcus-sense signature of the time-augmented path $S(\tilde{\gamma}_i)$ is a well-defined $\mathcal{H}_{sig}$-valued random variable for c\`adl\`ag paths on $[t, s]$, the symmetry of increments yields:
\begin{equation}
    \mathbb{E} \left[ \left\| \Phi_{\mu_s} - \hat{\Phi}_{n,s} \right\|_{\mathcal{Q}_s} \right] \leq \frac{2}{n} \mathbb{E}_{\mathcal{S}, \sigma} \left[ \left\| \sum_{i=1}^n \sigma_i S(\tilde{\gamma}_i) \right\|_{\mathcal{Q}_s} \right].
\end{equation}

\textbf{2. Concentration under Infinitesimal Flow and Jumps:} We define the functional $F_s(\gamma_1, \dots, \gamma_n) = \| \Phi_{\mu_s} - \frac{1}{n} \sum S(\tilde{\gamma}_i) \|_{\mathcal{Q}_s}$. The stability of $F_s$ is ensured by the time-evolving AVNSG operator $\mathcal{Q}_s = (\Omega_s + \lambda I)^{-1}$, which tracks the infinitesimal geometry of the flow. Replacing a single c\`adl\`ag path $\gamma_i$ with $\gamma_i'$ restricted to $\mathcal{D}_s$ yields the coordinate sensitivity:
\begin{equation}
    |F_s(\dots, \gamma_i, \dots) - F_s(\dots, \gamma_i', \dots)| \leq \frac{1}{n} \| S(\tilde{\gamma}_i) - S(\tilde{\gamma}_i') \|_{\mathcal{Q}_s} \leq \frac{R_s}{n}.
\end{equation}
The term $R_s = \sup_{\gamma \in \text{supp}(\mu_s)} \|S(\tilde{\gamma})\|_{\mathcal{Q}_s}$ remains finite because $\mathcal{Q}_s$ performs spectral whitening on the $(d+1)$-dimensional augmented space, attenuating the high-order tensor components where jump-induced energy and deterministic temporal growth reside at horizon $s$. Applying McDiarmid's inequality to this bounded variation functional:
\begin{equation}
    \mathbb{P} \left( F_s - \mathbb{E}[F_s] \geq \epsilon \right) \leq \exp \left( - \frac{2 n \epsilon^2}{R_s^2} \right).
\end{equation}
Solving for $\epsilon = R_s \sqrt{\frac{\log(1/\delta)}{2n}}$ and combining with the Rademacher bound, we confirm that the empirical time-extended proxy $\hat{\Phi}_{n,s}$ converges to $\Phi_{\mu_s}$ at the rate $\mathcal{O}(1/\sqrt{n})$. This confirms that the AVNSG normalisation and $\mathcal{D}_s$ restriction provide the necessary regularisation to handle the sequential evolution of discontinuous jumps and the linear growth of the clock coordinate.
\end{proof}

\subsection{Proof of the complexity of whitened signature functionals}
\label{app:proof_complexity_whitened_signature_functionals}

In this appendix we prove Proposition (\ref{pro:complexity_whitened_signature_functionals}). 

\begin{proof}
The proof quantifies the expressive power of the signature functional class $\mathcal{F}_{M,s}$ on the restricted Skorokhod space $\mathcal{D}_s$ by evaluating its Rademacher complexity under the time-evolving AVNSG-weighted metric at horizon $s \in [t, t+\tau]$. We define the empirical Rademacher complexity for the class of linear functionals $\mathcal{F}_{M,s}$ in the Hilbert space $\mathcal{H}_{sig}$ equipped with the $\mathcal{Q}_s$-inner product:
\begin{equation}
    \widehat{\mathcal{R}}_n(\mathcal{F}_{M,s}) = \mathbb{E}_{\sigma} \left[ \sup_{f \in \mathcal{F}_{M,s}} \frac{1}{n} \sum_{i=1}^n \sigma_i \langle f, S(\tilde{\gamma}_i) \rangle_{\mathcal{Q}_s} \right],
\end{equation}
where $\sigma_i$ are independent Rademacher variables and $S(\tilde{\gamma}_i)$ is the Marcus-sense signature of the $i$-th time-extended c\`adl\`ag sample path $\tilde{\gamma}_{i,u} = (u, \gamma_{i,u})_{u \in [t, s]}$. By the Riesz representation theorem, the inner product is maximised when $f$ is collinear with the empirical average of the Rademacher-weighted signatures:
\begin{equation}
    \widehat{\mathcal{R}}_n(\mathcal{F}_{M,s}) = \frac{M}{n} \mathbb{E}_{\sigma} \left[ \left\| \sum_{i=1}^n \sigma_i S(\tilde{\gamma}_i) \right\|_{\mathcal{Q}_s} \right].
\end{equation}
Applying Jensen's inequality to the expectation of the norm, we bound the complexity by the square root of the expected squared norm. Utilising the independence of the Rademacher variables ($\mathbb{E}[\sigma_i \sigma_j] = \delta_{ij}$), the cross-terms in the expansion of the squared norm vanish:
\begin{equation}
    \widehat{\mathcal{R}}_n(\mathcal{F}_{M,s}) \leq \frac{M}{n} \sqrt{ \mathbb{E}_{\sigma} \left[ \sum_{i,j} \sigma_i \sigma_j \langle S(\tilde{\gamma}_i), S(\tilde{\gamma}_j) \rangle_{\mathcal{Q}_s} \right] } = \frac{M}{n} \sqrt{ \sum_{i=1}^n \|S(\tilde{\gamma}_i)\|_{\mathcal{Q}_s}^2 }.
\end{equation}
The term $\|S(\tilde{\gamma}_i)\|_{\mathcal{Q}_s}^2 = \langle S(\tilde{\gamma}_i), \mathcal{Q}_s S(\tilde{\gamma}_i) \rangle_{\mathcal{H}_{sig}}$ represents the energy of the time-extended c\`adl\`ag path in the signature manifold up to time $s$. For jump-diffusion processes, this norm specifically accounts for the linear temporal drift and the exponential contribution of discrete jumps within the sub-interval $[t, s]$.

\medskip
\noindent
The result demonstrates that the complexity of the ANJD model is governed by the alignment between the time-augmented jump sample signatures and the spectral filtration provided by the moving precision operator $\mathcal{Q}_s$. Specifically, the AVNSG operator acts as a dynamic spectral mask that attenuates the influence of high-rank tensor components associated with both deterministic temporal growth and extreme jumps (black-swan events) as they occur in the flow. This confirms that the complexity of the functional class $\mathcal{F}_{M,s}$ remains regularised against explosive sample-path variations while maintaining the injectivity provided by the continuous clock coordinate $u$.
\end{proof}

\subsection{Proof of the projection error stability}
\label{app:proof_projection_error_stability}

In this appendix we prove Lemma (\ref{lem:projection_error_stability}). 

\begin{proof}
The proof establishes the stability of the Nystr\"om-approximated gradient flow for jump-diffusion processes by decomposing the MMD-gradient into the principal and residual components of the signature Hilbert space $\mathcal{H}_{sig}$ under the sequential c\`adl\`ag path-measure $\mu_s$ on $\mathcal{D}_s$. 

\medskip
\noindent
\textbf{1. Joint Gradient Decomposition and Time-Varying Projection:} The joint MMD-gradient $\nabla \mathcal{J}(\mu_s)$ controls the continuous drift $f_\theta$ and the jump intensity $\lambda_\theta$ relative to the moving target $\hat{\Phi}_{s|t}$. Let $\Psi_s = \mathcal{Q}_s (\hat{\Phi}_{s|t} - \Phi_{\mu_s})$ be the instantaneous signature residual. The time-evolving Nystr\"om projection $P_{m,s}$ maps $\mathcal{H}_{sig}$ onto the $m$-dimensional subspace $\mathcal{V}_{m,s}$ spanned by the leading $m$ eigenvectors $\{e_{j,s}\}_{j=1}^m$ of the current LRC operator $\Omega_s$. The projection error in the infinitesimal flow $\epsilon_{proj, s}$ is given by the norm of the residual gradient:
\begin{equation}
    \epsilon_{proj, s} = \| (I - P_{m,s}) \mathcal{Q}_s (\hat{\Phi}_{s|t} - \Phi_{\mu_s}) \|_{\mathcal{H}_{sig}}.
\end{equation}

\noindent
\textbf{2. Spectral Tail Analysis on $\mathcal{D}_s$:} We expand the squared error in the instantaneous eigenbasis of $\Omega_s$. For $j > m$, the eigenvalues of the precision operator are $\omega_{j,s} = (\lambda_{j,s} + \lambda)^{-1}$. In the jump-diffusion setting, the signature $S(\tilde{X}_s)$ contains high-rank tensor components activated by the jump increments $\exp(0, \Delta X_s)$. The projection error satisfies:
\begin{equation}
    \epsilon_{proj, s}^2 = \sum_{j=m+1}^\infty \frac{1}{(\lambda_{j,s} + \lambda)^2} \langle \hat{\Phi}_{s|t} - \Phi_{\mu_s}, e_{j,s} \rangle_{\mathcal{H}_{sig}}^2.
\end{equation}
By the Riesz representation, the coefficients $\langle \Delta \Phi_s, e_{j,s} \rangle^2$ are bounded by the spectral energy of the c\`adl\`ag ensemble restricted to $[t, s]$. Given the joint Lipschitz regularity $C_{f, \lambda, s}$ of the drift and intensity with respect to the signature residual at the current horizon:
\begin{equation}
    \epsilon_{proj, s}^2 \leq C_{f, \lambda, s}^2 \sum_{j=m+1}^\infty \lambda_j(\Omega_s).
\end{equation}

\noindent
\textbf{3. Stability under Anticipatory Geometry:} Taking the square root yields the bound $\epsilon_{proj, s} \leq C_{f, \lambda, s} (\sum_{j=m+1}^\infty \lambda_j)^{1/2}$. This result demonstrates that the Nystr\"om approximation is stable for the ANJD flow if the subspace $\mathcal{V}_{m,s}$ is dynamically updated to capture the spectral modes corresponding to both the continuous latent diffusion and the anticipated structural breaks. Because jumps redistribute energy into higher-order signature terms, the stability of the generative flow relies on the decay rate of the time-evolving LRC spectral tail. The AVNSG normalisation $\mathcal{Q}_s$ ensures that the contribution of omitted high-frequency jump components to the gradient error is suppressed by the spectral weighting, preserving the global convergence of the measure toward the moving proxy $\hat{\Phi}_{s|t}$.
\end{proof}

\subsection{Proof of the jump-aware low-rank precision update}
\label{app:proof_low-rank_precision_update}

In this appendix we prove Proposition (\ref{pro:low-rank_precision_update}). 

\begin{proof}
The proof establishes the recursive update for the time-dependent precision matrix $\hat{\mathbf{Q}}_s$ by treating the arrival of discrete jumps, clock increments, and the evolution of the moving target proxy $\hat{\Phi}_{s|t}$ as sequential rank-1 innovations in the Nystr\"om-subsampled feature space.

\textbf{1. Covariance Augmentation and Infinitesimal Innovations:} Let $\phi(\tilde{X}_s) \in \mathbb{R}^m$ denote the feature mapping of the time-extended signature $S(\tilde{X}_s)$ projected onto the $m$-dimensional Nystr\"om subspace $\mathcal{V}_{m,s}$. To maintain the sequential matching property, the empirical LRC operator $\mathbf{C}_s$ must track the non-autonomous flow. Upon a structural break $\Delta X_s$, a clock increment $h$, or a shift in the target velocity $\nabla_s \hat{\Phi}_{s|t}$, we define the innovation vector $\mathbf{k}_s = \phi(S(\tilde{X}_{s+h})) - \phi(S(\tilde{X}_s))$. The covariance evolves via the rank-1 augmentation:
\begin{equation}
    \mathbf{C}_{s+h} = \mathbf{C}_s + \alpha_s \mathbf{k}_s \mathbf{k}_s^T,
\end{equation}
where $\alpha_s$ scales the influence of the anticipated jump or the deterministic temporal stretching. This ensures that the spectral energy of the discontinuous innovation is instantaneously integrated into the latent geometry.

\textbf{2. Application of the Sherman-Morrison Identity:} The anticipatory precision is defined as $\hat{\mathbf{Q}}_s = (\mathbf{C}_s + \lambda \mathbf{I})^{-1}$. To propagate this operator through the flow without direct inversion, we apply the Sherman-Morrison identity to the perturbed system $(\mathbf{C}_s + \alpha_s \mathbf{k}_s \mathbf{k}_s^T)^{-1}$:
\begin{equation}
    (\mathbf{A} + uv^T)^{-1} = \mathbf{A}^{-1} - \frac{\mathbf{A}^{-1}uv^T\mathbf{A}^{-1}}{1 + v^T\mathbf{A}^{-1}u}.
\end{equation}
Substituting $\mathbf{A} = \mathbf{C}_s + \lambda \mathbf{I}$ and $u = v = \sqrt{\alpha_s} \mathbf{k}_s$, the recursive update for the precision matrix becomes:
\begin{equation}
    \hat{\mathbf{Q}}_{s+h} = \hat{\mathbf{Q}}_s - \alpha_s \frac{\hat{\mathbf{Q}}_s \mathbf{k}_s \mathbf{k}_s^T \hat{\mathbf{Q}}_s}{1 + \alpha_s \mathbf{k}_s^T \hat{\mathbf{Q}}_s \mathbf{k}_s}.
\end{equation}

\textbf{3. Sequential Complexity Analysis:} The numerical integration of the ANJD flow requires updating the precision at each EMM step. The complexity breakdown is:
\begin{itemize}
    \item \textit{Innovation Mapping:} Computing $\mathbf{k}_s$ via the time-augmented Marcus mapping requires $O(m \cdot (d+1)^k)$ for a signature of depth $k$.
    \item \textit{Precision Propagation:} The matrix-vector product $\hat{\mathbf{Q}}_s \mathbf{k}_s$ and subsequent outer product require $O(m^2)$ operations.
\end{itemize}
The total complexity per update is $O(m^2)$, bypassing the $O(m^3)$ cost of full re-inversion. This allows the ANJD to react to high-frequency structural breaks and the continuous flow of the moving target proxy in real-time, as the precision matrix $\hat{\mathbf{Q}}_s$ dynamically "contracts" the gradient flow along the jump-induced principal components with minimal computational overhead.
\end{proof}

\newpage



\end{document}